\documentclass{LMCS}

\usepackage{amssymb,amsmath,stmaryrd,xspace,url,epic,eepic,epsfig}
\usepackage{boxedminipage}
\usepackage{enumerate,hyperref}

\newcommand{\role}{\mathit{role}}
\newcommand{\At}{\mathit{At}}
\newcommand{\LV}{\mathit{LVar}}
\newcommand{\RV}{\mathit{RVar}}
\newcommand{\LA}{\mathit{LAto}}
\newcommand{\RA}{\mathit{RAto}}
\newcommand{\AC}{\mathit{AC}}
\newcommand{\eqAC}{=_{\mathit{AC}}}
\newcommand{\con}{\mathit{con}}

\newcommand{\INST}{\leq\!\!\stackrel{\scriptscriptstyle\bullet}{}}

\newcommand{\SLO}{\mathit{SLmO}}
\newcommand{\gtris}{\sqsubset}
\newcommand{\CDh}{^{[C/\widehat{D}]}}

\def\doi{6 (3:17) 2010}
\lmcsheading%
{\doi}
{1--31}
{}
{}
{Oct.~21, 2009}
{Sep.~\phantom04, 2010}
{}

\begin{document}

\title{Unification in the Description Logic $\mathcal{EL}$} 

\author[F.~Baader]{Franz Baader}
\address{Theoretical Computer Science, TU Dresden, Germany}
\email{\{baader,morawska\}@tcs.inf.tu-dresden.de} 
\author[B.~Morawska]{Barbara Morawska}
\address{\vskip-6 pt}
\thanks{Supported by DFG under grant BA 1122/14--1}


\begin{abstract}
The Description Logic $\mathcal{EL}$ has recently drawn considerable
attention since, on the one hand, important inference problems such as
the subsumption problem are polynomial. On the other hand, $\mathcal{EL}$
is used to define large biomedical ontologies. Unification in
Description Logics has been proposed as a novel inference
service that can, for example, be used to detect redundancies
in ontologies. 
The main result of this paper is that unification in $\mathcal{EL}$ 
is decidable. More precisely, $\mathcal{EL}$-unification is
NP-complete, and thus has the same complexity as $\mathcal{EL}$-matching. 
We also show that, w.r.t. the unification
type, $\mathcal{EL}$ is less well-behaved:
it is of type zero, which in particular
implies that there are unification problems that have no finite
complete set of unifiers.
\end{abstract}

\maketitle

\section{Introduction}

Description logics (DLs) \cite{BCNMP03} are a family of logic-based
knowledge representation formalisms, which can be used to represent
the conceptual knowledge of an application domain in a structured
and formally well-understood way. They are employed in various application
domains, such as natural language processing, configuration of technical systems,
databases, and biomedical ontologies, but their most notable success so far is
the adoption of the DL-based language OWL \cite{HoPH03}
as standard ontology language for the
semantic web.

In DLs, concepts are formally described by
{\em concept terms\/}, i.e., expressions that
are built from concept names (unary predicates) and role names
(binary predicates) using concept constructors. The expressivity
of a particular DL is determined by which concept constructors
are available in it.
From a semantic point of view, concept names and concept terms represent sets
of individuals, whereas roles represent binary relations between individuals.
For example, using the concept name $\mathsf{Woman}$,
and the role name $\mathsf{child}$, the concept of {\em women having
a daughter\/} can be represented by the concept term
\[
\mathsf{Woman}\sqcap\exists\, \mathsf{child}.\mathsf{Woman},
\]
and the concept of {\em women having only daughters\/} by
\[
\mathsf{Woman}\sqcap\forall\, \mathsf{child}.\mathsf{Woman}.
\]
Knowledge representation systems based on DLs
provide their users with various inference services that allow them
to deduce implicit knowledge from the explicitly represented knowledge.
An important inference problem solved by DL systems is the subsumption problem:
the subsumption algorithm allows one to determine subconcept-superconcept relationships.  
For example, the concept
term $\mathsf{Woman}$ subsumes the concept term 
$\mathsf{Woman}\sqcap\exists\, \mathsf{child}.\mathsf{Woman}$ since all
instances of the second term are also instances of the first
term, i.e., the second term is always interpreted as a
subset of the first term. With the help of the subsumption
algorithm, a newly introduced concept term can automatically be
placed at the correct position in the hierarchy of the already
existing concept terms.

Two concept terms $C, D$ are {\em equivalent\/} ($C\equiv D$) if they subsume each other,
i.e., if they always represent the same set of individuals.
For example, the terms
$\forall\,\mathsf{child}.\mathsf{Rich}\sqcap\forall\, \mathsf{child}.\mathsf{Woman}$ and
$\forall\, \mathsf{child}.(\mathsf{Rich}\sqcap\mathsf{Woman})$ are
equivalent since the value restriction operator ($\forall\, r.C$)
distributes over the conjunction operator ($\sqcap$). If we replace the
value restriction operator by the existential restriction operator ($\exists\, r.C$),
then this equivalence no longer holds. However, for this operator, we still have
the equivalence 
$$
\exists\,\mathsf{child}.\mathsf{Rich}\sqcap\exists\,\mathsf{child}.(\mathsf{Woman}\sqcap\mathsf{Rich}) \equiv \exists\,\mathsf{child}.(\mathsf{Woman}\sqcap\mathsf{Rich}).
$$
The equivalence test can, for example, be used to find out whether a
concept term representing a particular notion has already been introduced,
thus avoiding multiple introduction of the same concept into the
concept hierarchy. This inference capability is very important if the knowledge
base containing the concept terms is very large, evolves during a long time
period, and is extended and maintained by several knowledge engineers.
However, testing for equivalence of concepts is not always sufficient to
find out whether, for a given concept term, there already exists another
concept term in the knowledge base describing the same notion. On the one hand,
different knowledge engineers may use different names for concepts, like $\mathsf{Male}$ versus
$\mathsf{Masculine}$. On the other hand, they may model on different levels
of granularity.
For example, assume that one knowledge engineer has defined the concept of 
{\em men loving fast cars\/} by the concept term
\[
\mathsf{Human}\sqcap\mathsf{Male}\sqcap \exists\,\mathsf{loves}.\mathsf{Sports\_car}.
\]
A second knowledge engineer might represent this notion in a somewhat different
way, e.g., by using the concept term
\[
\mathsf{Man}\sqcap \exists\,\mathsf{loves}.(\mathsf{Car}\sqcap\mathsf{Fast}).
\]
These two concept terms
are not equivalent, but they are meant to represent the same concept.
The two terms can obviously be made equivalent by substituting the concept name
$\mathsf{Sports\_car}$ in the first term by the concept term $\mathsf{Car}\sqcap\mathsf{Fast}$
and the concept name $\mathsf{Man}$ in the second term by the concept term 
$\mathsf{Human}\sqcap\mathsf{Male}$. 
This leads us to {\em unification of concept terms\/},
i.e., the question whether two concept terms can be made equivalent by applying
an appropriate substitution, where a substitution
replaces (some of the) concept names
by concept terms.
Of course, it is not necessarily the case that unifiable concept terms are
meant to
represent the same notion. A unifiability test can, however, suggest to the knowledge
engineer possible candidate terms.
A \emph{unifier} (i.e., a substitution whose application makes the two terms equivalent) then
proposes appropriate definitions for the concept names. In our example, we know that,
if we define $\mathsf{Man}$ as $\mathsf{Human}\sqcap\mathsf{Male}$ and $\mathsf{Sports\_car}$
as $\mathsf{Car}\sqcap\mathsf{Fast}$, then the concept terms 
$\mathsf{Human}\sqcap\mathsf{Male}\sqcap \exists\,\mathsf{loves}.\mathsf{Sports\_car}$
and
$\mathsf{Man}\sqcap \exists\,\mathsf{loves}.(\mathsf{Car}\sqcap\mathsf{Fast})$
are equivalent w.r.t.\ these definitions.

Unification in DLs was first considered in \cite{BaNa00} for a DL called 
$\mathcal{ FL}_0$, which has the concept constructors \emph{conjunction} ($\sqcap$),
\emph{value restriction} ($\forall\, r.C$), and the \emph{top concept} ($\top$). It was
shown that unification in $\mathcal{ FL}_0$ is decidable and ExpTime-complete, i.e.,
given an $\mathcal{ FL}_0$-unification problem, we can effectively decide whether
it has a solution or not, but in the worst-case, any such decision procedure
needs exponential time. This result was extended in \cite{BaKu01} to a more expressive
DL, which additionally has the role constructor \emph{transitive closure}. 
Interestingly, the \emph{unification type} of $\mathcal{ FL}_0$ had been determined
almost a decade earlier in \cite{Baad89}. In fact, as shown in \cite{BaNa00},
unification in $\mathcal{ FL}_0$ corresponds to unification modulo the
equational theory of idempotent Abelian monoids with several homomorphisms.
In \cite{Baad89} it was shown that, already for a single homomorphism, unification
modulo this theory has unification type zero, i.e., there are unification
problems for this theory that do not have a minimal complete set of unifiers.
In particular, such unification problems cannot have a finite complete
set of unifiers.

In this paper, we consider unification in the DL $\mathcal{ EL}$.
The $\mathcal{EL}$-family consists of inexpressive
DLs whose main distinguishing feature is that they
provide their users with \emph{existential restrictions} ($\exists\, r.C$) rather than value restrictions
($\forall\, r.C$)
as the main concept constructor involving roles. The core language of this family
is $\mathcal{EL}$, which has the top concept, conjunction,
and existential restrictions as concept constructors.
This family has recently drawn considerable attention since, on the one hand,
the subsumption problem stays tractable (i.e., decidable in polynomial time)
in situations where $\mathcal{ FL}_0$, the corresponding DL with value restrictions, 
becomes intractable: subsumption between concept terms is tractable
for both $\mathcal{ FL}_0$ and $\mathcal{EL}$ \cite{LeBr85,BaKM99}, 
but allowing the use of concept definitions
or even more expressive terminological formalisms makes $\mathcal{ FL}_0$ intractable
\cite{Nebe90,Baad90c,KaNi03,BaBL05}, whereas it leaves $\mathcal{EL}$ tractable \cite{Baad03e,Bran04,BaBL05}.
On the other hand, although of limited expressive power, $\mathcal{EL}$ is nevertheless used
in applications, e.g., to define biomedical ontologies. For example, both the
large medical ontology {\sc Snomed~ct}\footnote{%
http://www.ihtsdo.org/snomed-ct/
}
%\cite{Spackman2000} 
and the Gene Ontology\footnote{%
http://www.geneontology.org/
}
%\cite{go-consorcium} 
can be expressed in $\mathcal{EL}$, and the same is
true for large parts of the medical ontology {\sc Galen} \cite{ReHo97}.
The importance of $\mathcal{EL}$ can also be seen from the fact that the new 
OWL\,2 standard\footnote{%
See http://www.w3.org/TR/owl2-profiles/
}
contains a sub-profile OWL\,2\,EL, which is based on (an extension of) $\mathcal{EL}$.

Unification in $\mathcal{EL}$ has, to the best of our knowledge, not been investigated
before, but matching (where one side of the equation(s) to be solved does not contain variables)
has been considered in \cite{BaKu00,Kues01}. In particular, it was shown
in \cite{Kues01} that the decision problem, i.e., the problem of deciding whether
a given $\mathcal{EL}$-matching problem has a matcher or not, is NP-complete. Interestingly,
$\mathcal{ FL}_0$ behaves better w.r.t.\ matching than $\mathcal{EL}$: for $\mathcal{ FL}_0$, the decision
problem is tractable \cite{BKBM99}. In this paper, we show that, w.r.t.\ the unification
type, $\mathcal{ FL}_0$ and $\mathcal{EL}$ behave the same: just as $\mathcal{ FL}_0$, the DL $\mathcal{EL}$
has unification type zero. However, w.r.t.\ the decision problem, $\mathcal{EL}$ behaves
much better than $\mathcal{ FL}_0$: $\mathcal{EL}$-unification is NP-complete, 
and thus has the same complexity as $\mathcal{EL}$-matching.

Regarding unification in DLs that are more expressive than $\mathcal{EL}$ and $\mathcal{ FL}_0$,
one must look at the literature on unification in modal logics.
It is well-known that there is a close connection between modal logics and DLs \cite{BCNMP03}.
For example, the DL $\mathcal{ALC}$, which can be obtained by adding negation to $\mathcal{EL}$
or $\mathcal{FL}_0$, corresponds to the basic (multi-)modal logic $\mathsf{K}$. Decidability of
unification in $\mathsf{K}$ is a long-standing open problem. Recently, undecidability of unification
in some extensions of $\mathsf{K}$ (for example, by the universal modality) was shown in \cite{WoZa08}.
The undecidability results in \cite{WoZa08} also imply undecidability of unification in some expressive
DLs (e.g., $\mathcal{SHIQ}$ \cite{HoST00}).
The unification types of some modal (and related) logics have been determined by Ghilardi; for example
in \cite{Ghil00} he shows that $\mathsf{K}4$ and $\mathsf{S}4$ have unification type finitary.
Unification in sub-Boolean modal logics (i.e., modal logics that are not closed under all Boolean operations,
such as the modal logics corresponding to $\mathcal{EL}$ and $\mathcal{FL}_0$)
has, to the best of our knowledge, not been considered in the modal logic literature.

In addition to unification of concept terms as introduced until now, 
we will also consider unification w.r.t.\ a so-called acyclic TBox in this article.
Until now, we have only talked about concept terms, i.e., complex descriptions of concepts that
are built from concept and role names using the concept constructors of the given DL. In applications
of DLs, it is, of course, inconvenient to always use such complex descriptions when referring to concepts.
For this reason, DLs are usually also equipped with a terminological formalism. In its simplest
form, this formalism allows to introduce abbreviations for concept terms. For example, the two
concept definitions
$$
\mathsf{Mother} \equiv \mathsf{Woman}\sqcap\exists\,\mathsf{child}.\mathsf{Human}\ \ \mbox{and}\ \
\mathsf{Woman} \equiv \mathsf{Human}\sqcap\mathsf{Female}
$$
introduce the abbreviation $\mathsf{Woman}$ for the concept term $\mathsf{Human}\sqcap\mathsf{Female}$
and the abbreviation $\mathsf{Mother}$ for the concept term
$\mathsf{Human}\sqcap\mathsf{Female}\sqcap\exists\,\mathsf{child}.\mathsf{Human}$.
A finite set of such concept definitions is called an \emph{acyclic} TBox if it is unambiguous
(i.e., every concept name occurs at most once as left-hand side) and acyclic (i.e., there are no cyclic 
dependencies between concept definitions). These restrictions
ensure that every defined concept (i.e., concept name occurring on the left-hand side of a definition)
has a unique expansion to a concept term that it abbreviates. Inference problems like subsumption
and unification can also be considered w.r.t.\ such acyclic TBoxes. As mentioned above, the complexity
of the subsumption problem increases for the DL $\mathcal{ FL}_0$ if acyclic TBoxes are taken into
account \cite{Nebe90}. In contrast, for $\mathcal{EL}$, the complexity of the subsumption problem 
stays polynomial in the presence of acyclic TBoxes. We show that, for unification in $\mathcal{EL}$,
adding acyclic TBoxes is also harmless, i.e., unification 
in  $\mathcal{EL}$ \emph{w.r.t.\ acyclic TBoxes} is also NP-complete.

This article is structured as follows.
In the next section, we define the DL $\mathcal{EL}$ and unification in $\mathcal{EL}$
more formally. In Section~\ref{subs}, we recall the characterization of subsumption
and equivalence
in $\mathcal{EL}$ from \cite{Kues01}, and in Section~\ref{type:zero} we
use this to show that unification in $\mathcal{EL}$ has type zero. 
In Section~\ref{decision:problem}, we show that unification in  $\mathcal{EL}$
is NP-complete. 
The unification algorithm establishing the complexity upper bound is a typical ``guess and then test''
NP-algorithm, and thus it is unlikely that a direct implementation of this algorithm
will perform well in practice. In Section~\ref{goal:oriented:sect}, we introduce a more goal-oriented
unification algorithm for $\mathcal{EL}$, in which non-deterministic decisions are only made
if they are triggered by ``unsolved parts'' of the unification problem.
In Section~\ref{eq:th:sect},
we point out that our results for  $\mathcal{EL}$-unification imply that unification modulo the equational
theory of semilattices with monotone operators \cite{SoSt08}
is NP-complete and of unification type zero.
 
More information about Description Logics can be found in \cite{BCNMP03}, and
about unification theory in \cite{BaSn01}. This article is an extended version of a paper \cite{BaMo09}
published in the proceedings of the 20th international Conference on Rewriting Techniques and
applications (RTA'09). In addition to giving more detailed proofs, we have added
the goal-oriented unification algorithm (Section~\ref{goal:oriented:sect})
and the treatment of unification modulo acyclic TBoxes (Subsection~\ref{unif:mod:TBox}).

\section{Unification in $\mathcal{EL}$}

In this section, we first
define the syntax and semantics of $\mathcal{EL}$-concept terms as well as
the subsumption and the equivalence relation on these terms. Then, we introduce unification
of $\mathcal{EL}$-concept terms, and finally extend this notion to unification modulo 
an acyclic TBox.
 
\subsection{The Description Logic $\mathcal{EL}$}

Starting with a set $N_{\con}$ of concept names and a set $N_{\role}$
of role names, \emph{$\mathcal{EL}$-concept terms} are built using
the following concept constructors:
the nullary constructor \emph{top-concept} ($\top$),
the binary constructor \emph{conjunction} ($C\sqcap D$), and for every role name $r\in N_{\role}$,
the unary constructor \emph{existential restriction} ($\exists\, r.C$).
The semantics of $\mathcal{EL}$
is defined in the usual way, using the notion of an interpretation
$\mathcal{ I } = (\mathcal D_\mathcal{ I},{\cdot}^\mathcal{ I})$,
which consists of a nonempty domain $\mathcal  D_\mathcal{ I}$ and an interpretation
function ${\cdot}^\mathcal{ I}$ that assigns binary relations on $\mathcal  D_\mathcal{ I}$
to role names and subsets of $\mathcal  D_\mathcal{ I}$ to concept terms,
as shown in the semantics column of Table~\ref{el-syn-sem}.

\begin{table}[t] %\label{el-syn-sem}
\begin{center}
\begin{tabular}{|l|c|c|}
 \hline Name &
Syntax & Semantics\\
\hline \hline
\rule{0cm}{.95em}concept name & $A$ & $A^{\mathcal I}\subseteq \mathcal  D_\mathcal{ I}$\\
\hline
\rule{0cm}{.95em}role name & $r$ & $r^{\mathcal I}\subseteq \mathcal  D_\mathcal{ I} \times \mathcal  D_\mathcal{ I}$\\
\hline
\rule{0cm}{.95em}top-concept & $\top$ & $\top^{\mathcal I} =  \mathcal  D_\mathcal{ I}$\\
\hline
\rule{0cm}{.95em}conjunction & $C \sqcap D$ & $(C \sqcap D) ^{\mathcal I} =  C^{\mathcal I}\cap D^{\mathcal I}$\\
\hline
\rule{0cm}{.95em}existential restriction & $\exists\, r.C$ & $(\exists\, r.C)^{\mathcal I} =
   \{x \mid \exists\, y : (x,y) \in r^{\mathcal I} \land y \in C^{\mathcal I}\}$\\
\hline \hline
\rule{0cm}{.95em}subsumption &$C \sqsubseteq D$& $C^{\mathcal I} \subseteq D^{\mathcal I}$\\
\hline
\rule{0cm}{.95em}equivalence &$C \equiv D$& $C^{\mathcal I} = D^{\mathcal I}$\\
\hline
\end{tabular}
\end{center}
\caption{Syntax and semantics of $\mathcal {EL}$}
\label{el-syn-sem}
\end{table}

The concept term $C$ \emph{is subsumed by} the concept term $D$ (written $C\sqsubseteq D$)
iff $C^{\mathcal I} \subseteq D^{\mathcal I}$ holds for all interpretations $\mathcal I$.
We say that $C$ \emph{is equivalent to} $D$ (written $C\equiv D$) iff $C\sqsubseteq D$ and
$D \sqsubseteq C$, i.e., iff $C^{\mathcal I} = D^{\mathcal I}$ holds for all 
interpretations $\mathcal I$. The concept term $C$ \emph{is strictly subsumed by}
the concept term $D$ (written $C\sqsubset D$) iff $C\sqsubseteq D$ and $C \not\equiv D$.
It is well-known 
that subsumption (and thus also equivalence) of $\mathcal{ EL}$-concept terms
can be decided in polynomial time \cite{BaKM99}.

\subsection{Unification of concept terms}

In order to define unification of concept terms, we first introduce the notion of a substitution
operating on concept terms. To this purpose, we partition the set of concepts names
into a set $N_v$ of concept variables (which may be replaced by substitutions)
and a set $N_c$ of concept constants (which must not be replaced by substitutions).
Intuitively, $N_v$ are the concept names that have possibly been given another name
or been specified in more detail in another concept term describing the same notion.
The elements of $N_c$ are the ones of which it is assumed that the same name
is used by all knowledge engineers (e.g., standardized names in a certain domain).

A {\em substitution\/} $\sigma$ is a mapping from $N_v$ into the set of all
$\mathcal{ EL}$-concept terms. This mapping is extended to concept terms in the obvious way,
i.e.,
\begin{enumerate}[$\bullet$]
\item $\sigma(A) := A$ for all $A\in N_c$,
\item $\sigma(\top) := \top$,
\item $\sigma(C\sqcap D) := \sigma(C) \sqcap \sigma(D)$, and
\item $\sigma(\exists\, r . C) := \exists\, r .\sigma(C)$.
\end{enumerate}

\begin{defi}
An \emph{$\mathcal{ EL}$-unification problem} is of the form 
$\Gamma = \{ C_1\equiv^? D_1,\ldots,C_n\equiv^? D_n\}$, 
where
$C_1, D_1,\ldots, C_n, D_n$ are $\mathcal{ EL}$-concept terms. The substitution $\sigma$
is a {\em unifier} (or \emph{solution}) of $\Gamma$ iff $\sigma(C_i) \equiv \sigma(D_i)$
for $i = 1,\ldots, n$.
In this case, $\Gamma$ is called \emph{solvable} or
{\em unifiable\/}.
\end{defi}

When we say that $\mathcal{ EL}$-unification is \emph{decidable}, then
we mean that the following decision problem is decidable: given an
$\mathcal{ EL}$-unification problem $\Gamma$, decide whether $\Gamma$ is solvable or not.
Accordingly, we say that $\mathcal{ EL}$-unification is \emph{NP-complete} if this decision problem
is NP-complete.

In the following, we introduce some standard notions from unification theory \cite{BaSn01},
but formulated for the special case of $\mathcal{ EL}$-unification rather than for an
arbitrary equational theory.
Unifiers can be compared using the instantiation preorder $\INST$.
Let $\Gamma$ be an $\mathcal{ EL}$-unification problem, $V$ the set of
variables occurring in $\Gamma$, and $\sigma, \theta$ two unifiers of this problem.
We define
$$
\sigma\INST \theta\ \ \mbox{iff}\ \ \mbox{there is a substitution $\lambda$ such that}\
\theta(X) \equiv \lambda(\sigma(X))\ \mbox{for all}\ X\in V.
$$
If $\sigma\INST \theta$, then we say that $\theta$ is an \emph{instance} of $\sigma$.

\begin{defi}
Let $\Gamma$ be an $\mathcal{ EL}$-unification problem. The set of substitutions
$M$ is called a \emph{complete set of unifiers} for $\Gamma$ iff it satisfies
the following two properties:
\begin{enumerate}[(1)]
\item every element of $M$ is a unifier of $\Gamma$;
\item if $\theta$ is a unifier of $\Gamma$, then there exists a unifier
          $\sigma\in M$ such that $\sigma\INST \theta$.
\end{enumerate}
The set $M$ is called a \emph{minimal complete set of unifiers} for $\Gamma$ iff
it additionally satisfies
\begin{enumerate}[(1)]
\item[(3)] if $\sigma, \theta\in M$, then $\sigma\INST \theta$ implies $\sigma = \theta$.
\end{enumerate}
\end{defi}

The unification type of a given unification problem is determined by
the existence and cardinality\footnote{%
It is easy to see that the cardinality of a minimal complete set of unifiers is
uniquely determined by the unification problem.
} 
of such a minimal complete set. 

\begin{defi}
Let $\Gamma$ be an $\mathcal{ EL}$-unification problem. 
This problem has type 
\begin{enumerate}[$\bullet$]
\item 
{\em unitary\/} iff it has a minimal complete set of unifiers
of cardinality $1$;
\item
{\em finitary\/} iff it has a finite minimal complete set of unifiers;
\item
{\em infinitary\/}
iff it has an infinite minimal complete set of unifiers;
\item
{\em zero\/} iff it
does not have a minimal complete set of unifiers.
\end{enumerate}
\end{defi}
 
Note that the set of all unifiers of a given $\mathcal{ EL}$-unification problem
is always a complete set of unifiers. However, this set is usually infinite
and redundant (in the sense that some unifiers are instances of others).
For a unitary or finitary $\mathcal{ EL}$-unification problem, all unifiers
can be represented by a finite complete set of unifiers, whereas for problems of
type infinitary or zero this is no longer possible. In fact, if a problem
has a finite complete set of unifiers $M$, then it also has a finite \emph{minimal}
complete set of unifiers, which can be obtained by iteratively removing
redundant elements from $M$.
%, i.e., by removing the unifier $\theta\in M$ if it is
%an instance of another unifier in $M$. 
 %
For an infinite complete set of unifiers, this approach of removing redundant
unifiers may be infinite, and the set reached in the limit need no longer
be complete. This is what happens for problems of type zero.
The difference between infinitary
and type zero is that a unification problem of type zero cannot even have
a non-redundant complete set of unifiers, i.e., every complete set of
unifiers must contain different unifiers $\sigma, \theta$ such that $\sigma\INST \theta$.
More information on unification type zero can be found in \cite{Baad89b}.

When we say that \emph{$\mathcal{ EL}$ has unification type zero}, we mean that there
exists an $\mathcal{ EL}$-unification problem that has type zero. Before we can prove 
in Section~\ref{type:zero} that this is indeed the case, we must have a closer look at 
equivalence in $\mathcal{ EL}$ in Section~\ref{subs}. But first, we consider unification
modulo acyclic TBoxes.

\subsection{Unification modulo acyclic TBoxes}\label{unif:mod:TBox}

A \emph{concept definition} is of the form $A\doteq C$ where $A$ is a concept
name and $C$ is a concept term. A \emph{TBox} $\mathcal{T}$ is a finite set of concept definitions 
such that no concept name occurs more than once on the left-hand side of
a concept definition in $\mathcal{T}$. The TBox $\mathcal{T}$ is called \emph{acyclic}
if there are no cyclic dependencies between its concept definitions. To be more precise,
we say that the concept name $A$ \emph{directly depends on} the concept name $B$ in a TBox $\mathcal{T}$
if $\mathcal{T}$ contains a concept definition $A\doteq C$ and $B$ occurs in $C$. Let
\emph{depends on} be the transitive closure of the relation  \emph{directly depends on}.
Then $\mathcal{T}$ contains a \emph{terminological cycle} if there is a concept name $A$ that depends
on itself. Otherwise, $\mathcal{T}$ is called \emph{acyclic}.
Given a TBox $\mathcal{T}$, we call a concept name $A$ a \emph{defined concept} if it
occurs as the left-side of a concept definition $A\doteq C$ in $\mathcal{T}$. All
other concept names are called \emph{primitive concepts}.

The interpretation
${\mathcal I}$ is a model of the TBox ${\mathcal T}$ iff $A^{\mathcal I} = C^{\mathcal I}$
holds for all concept definitions $A\doteq C$ in ${\mathcal T}$. Subsumption and equivalence
w.r.t.\ a TBox are defined as follows: $C\sqsubseteq_{\mathcal T} D$ ($C\equiv_{\mathcal T} D$)
iff $C^{\mathcal I} \subseteq D^{\mathcal I}$ ($C^{\mathcal I} = D^{\mathcal I}$) holds for all
models $\mathcal I$ of $\mathcal T$. 
 
Subsumption and equivalence w.r.t.\ an acyclic TBox can be reduced to subsumption and equivalence
of concept terms (without TBox) by \emph{expanding} the concept terms w.r.t.\ the TBox:
given a concept term $C$, its expansion ${C}^{\mathcal{T}}$ w.r.t.\ the acyclic TBox
$\mathcal{T}$ is obtained by exhaustively
replacing all defined concept names $A$ occurring on the left-hand side of concept
definitions $A\doteq C$ in $\mathcal{T}$ by their defining concept terms $C$. 
Given concept terms $C, D$, we have $C\sqsubseteq_{\mathcal{T}} D$ 
%($C\equiv_{\mathcal{T}} D$)
iff $C^{\mathcal{T}}\sqsubseteq D^{\mathcal{T}}$ 
%($C^{\mathcal{T}}\equiv D^{\mathcal{T}}$) 
\cite{BaNu03}. The same is true for equivalence, i.e.,
$C\equiv_{\mathcal{T}} D$ iff $C^{\mathcal{T}}\equiv D^{\mathcal{T}}$.
This expansion process may, however, result in an exponential blow-up \cite{Nebe90,BaNu03}, and thus
this reduction of subsumption and equivalence w.r.t.\ an acyclic TBox to subsumption and equivalence
without a TBox is not polynomial. Nevertheless, in $\mathcal{ EL}$, subsumption (and thus also equivalence)
w.r.t.\ acyclic TBoxes can be decided in polynomial time \cite{Baad03e}.

In our definition of unification modulo acyclic TBoxes, we assume that all defined concepts are
concept constants. In fact, defined concepts already have a definition in the given TBox, and
thus it does not make sense to introduce new ones for them by unification. In this setting, a
{\em substitution\/} $\sigma$ is a mapping from $N_v$ into the set of all $\mathcal{ EL}$-concept 
terms \emph{not containing any defined concepts}.\footnote{%
This restriction prevents the unifier from introducing cycles into the TBox.
}
The extension of $\sigma$ to concept terms is defined as in the previous subsection, and its application
to $\mathcal T$ is defined as
$$
\sigma(\mathcal T) := \{ A\doteq \sigma(C)\mid A\doteq C\in \mathcal{T}\}.
$$
 
\begin{defi}
An \emph{$\mathcal{ EL}$-unification problem modulo an acyclic TBox} is of the form
$\Gamma = \{ C_1\equiv_{\mathcal{T}}^? D_1,\ldots,C_n\equiv_{\mathcal{T}}^? D_n\}$,
where
$C_1, D_1,\ldots, C_n, D_n$ are $\mathcal{ EL}$-concept terms, and $\mathcal{T}$
is an acyclic $\mathcal{ EL}$-TBox. The substitution $\sigma$
is a {\em unifier} (or \emph{solution}) of $\Gamma$ \emph{modulo $\mathcal{T}$} iff 
$\sigma(C_i) \equiv_{\sigma(\mathcal{T})} \sigma(D_i)$
for $i = 1,\ldots, n$.
In this case, $\Gamma$ is called \emph{solvable modulo $\mathcal{T}$} or
{\em unifiable modulo $\mathcal{T}$\/}.
\end{defi}
Coming back to our example from the introduction, assume that one knowledge
engineer has written the concept definition
\[
\mathsf{Real\_man} \doteq
\mathsf{Human}\sqcap\mathsf{Male}\sqcap \exists\,\mathsf{loves}.\mathsf{Sports\_car}.
\]
to the TBox, whereas a second one has written the definition
\[
\mathsf{Stupid\_man} \doteq
\mathsf{Man}\sqcap \exists\,\mathsf{loves}.(\mathsf{Car}\sqcap\mathsf{Fast}),
\]
where all the concept names occurring on the left-hand side of these definitions
are primitive concepts. Then the substitution that replaces $\mathsf{Sports\_car}$
by $\mathsf{Car}\sqcap\mathsf{Fast}$ and $\mathsf{Man}$ by $\mathsf{Human}\sqcap\mathsf{Male}$
is a unifier of $\{\mathsf{Real\_man}\equiv_{\mathcal{T}}^? \mathsf{Stupid\_man}\}$
w.r.t.\ the TBox $\mathcal{T}$ consisting of these two definitions.

Using expansion, we can reduce unification modulo an acyclic TBox to unification
without a TBox. In fact, the following lemma is an easy consequence of the fact
that $\sigma(C^{\mathcal{T}}) = \sigma(C)^{\sigma(\mathcal{T})}$ holds for all
$\mathcal{ EL}$-concept terms $C$.

\begin{lem}\label{expansion:lem}
The substitution $\sigma$ is a unifier of
$\{ C_1\equiv_{\mathcal{T}}^? D_1,\ldots,C_n\equiv_{\mathcal{T}}^? D_n\}$ modulo $\mathcal{T}$
iff it is a unifier of 
$\{ C_1^{\mathcal{T}}\equiv^? D_1^{\mathcal{T}},\ldots,C_n^{\mathcal{T}}\equiv^? D_n^{\mathcal{T}}\}$.
\end{lem}

Since expansion can cause an exponential blow-up, this is not a polynomial reduction. In the
remainder of this subsection, we show that there actually exists a polynomial-time reduction
of unification modulo an acyclic TBox to unification without a TBox. 

We say that the
$\mathcal{EL}$-unification problem $\Gamma$ is in \emph{dag-solved form}
if it can be written as $\Gamma = \{X_1\equiv^? C_1,\ldots, X_n\equiv^? C_n\}$,
where $X_1,\ldots,X_n$ are distinct concept variables such that, for all $i\leq n$, 
$X_i$ does not occur in $C_{i},\ldots,C_n$. For $i = 1,\ldots,n$, let $\sigma_i$ be
the substitution that maps $X_i$ to $C_i$ and leaves all other variables unchanged.
We define the substitution $\sigma_\Gamma$ as
$$
\sigma_\Gamma(X_i) := \sigma_n(\cdots(\sigma_i(X_i))\cdots)
$$
for $i = 1, \ldots,n$, and $\sigma_\Gamma(X) := X$ for all other variables $X$.
The following is an instance of a well-known fact from unification theory \cite{JoKi91}.

\begin{lem}\label{JK:lemma}
Let $\Gamma = \{X_1\equiv^? C_1,\ldots, X_n\equiv^? C_n\}$ 
be an $\mathcal{EL}$-unification problem in dag-solved form. Then, the set
$\{\sigma_\Gamma\}$ is a complete set of unifiers for $\Gamma$.
\end{lem}

There is a close relationship between acyclic TBoxes and unification problems
in dag-solved form. In fact, if $\mathcal{T}$ is an acyclic TBox, then there is an
enumeration $A_1,\ldots,A_n$ of the defined concepts in $\mathcal{T}$ such that
$\mathcal{T} = \{A_1\doteq C_1,\ldots, A_n\doteq C_n\}$ and $A_i$ does not occur
in $C_i,\ldots,C_n$. Consequently, the corresponding unification problem
$$
\Gamma(\mathcal{T}) := \{A_1\equiv^? C_1,\ldots, A_n\equiv^? C_n\}
$$
(where $A_1,\ldots, A_n$ are now viewed as concept variables) is in dag-solved form. 
In addition, it is easy to see that, for any $\mathcal{EL}$-concept term $C$, we have
$C^{\mathcal{T}} = \sigma_{\Gamma(\mathcal{T})}(C)$.

%Given an $\mathcal{EL}$-unification problem $\Gamma = \{X_1\equiv^? C_1,\ldots, X_n\equiv^? C_n\}$,
%the corresponding TBox
%$$
%\mathcal{T}_\Gamma := \{X_1\doteq C_1,\ldots, X_n\doteq C_n\}
%$$
%(where $X_1,\ldots, X_n$ are now viewed as concept variables) is obviously acyclic.
%In addition, it is easy to see that, for any $\mathcal{EL}$-concept term $C$, we have
%$\sigma_\Gamma(C) = C^{\mathcal{T}_\Gamma}$.

\begin{lem}
The $\mathcal{ EL}$-unification problem 
$\Gamma = \{ C_1\equiv_{\mathcal{T}}^? D_1,\ldots,C_n\equiv_{\mathcal{T}}^? D_n\}$ 
is solvable modulo the acyclic TBox $\mathcal{T}$ iff 
$\{ C_1\equiv^? D_1,\ldots,C_n\equiv^? D_n\}\cup \Gamma(\mathcal{T})$ is solvable.\footnote{%
Note that the defined concepts of $\mathcal{T}$ are treated as concept constants in
$\Gamma$, and as concept variables in $\{ C_1\equiv^? D_1,\ldots,C_n\equiv^? D_n\}\cup \Gamma(\mathcal{T})$.
}
\end{lem}

\begin{proof}
Assume that $\theta$ is a unifier of 
$\Gamma = \{ C_1\equiv_{\mathcal{T}}^? D_1,\ldots,C_n\equiv_{\mathcal{T}}^? D_n\}$ 
modulo $\mathcal{T}$. Then it is a unifier of 
$\widehat{\Gamma} := 
\{ C_1^{\mathcal{T}}\equiv^? D_1^{\mathcal{T}},\ldots,C_n^{\mathcal{T}}\equiv^? D_n^{\mathcal{T}}\}$,
by Lemma~\ref{expansion:lem}. 
Since $C_i^{\mathcal{T}} = \sigma_{\Gamma(\mathcal{T})}(C_i)$ and
$D_i^{\mathcal{T}} = \sigma_{\Gamma(\mathcal{T})}(D_i)$, we have
$\widehat{\Gamma} = \{ \sigma_{\Gamma(\mathcal{T})}(C_i) \equiv^? \sigma_{\Gamma(\mathcal{T})}(D_i)\mid
1\leq i\leq n\}$. Consequently, if we define the substitution $\tau$ by setting
$\tau(X) := \theta(\sigma_{\Gamma(\mathcal{T})}(X))$ for all concept variables and defined concepts
$X$, then $\tau$ is a unifier of $\{ C_1\equiv^? D_1,\ldots,C_n\equiv^? D_n\}$. In addition,
since $\sigma_{\Gamma(\mathcal{T})}$ is a unifier of $\Gamma(\mathcal{T})$, $\tau$ is also a unifier
of $\{ C_1\equiv^? D_1,\ldots,C_n\equiv^? D_n\}\cup \Gamma(\mathcal{T})$.

Conversely, assume that $\tau$ is a unifier of
$\{ C_1\equiv^? D_1,\ldots,C_n\equiv^? D_n\}\cup \Gamma(\mathcal{T})$.
In particular, this implies that $\tau$ is a unifier of $\Gamma(\mathcal{T})$.
By Lemma~\ref{JK:lemma}, $\{\sigma_{\Gamma(\mathcal{T})}\}$ is a complete set of unifiers 
for $\Gamma(\mathcal{T})$, and thus there is a substitution $\theta$ such that
$\tau(X) = \theta(\sigma_{\Gamma(\mathcal{T})}(X))$ for all concept variables occurring
in the unification problem $\{ C_1\equiv^? D_1,\ldots,C_n\equiv^? D_n\}\cup \Gamma(\mathcal{T})$.
Since $C_i^{\mathcal{T}} = \sigma_{\Gamma(\mathcal{T})}(C_i)$ and
$D_i^{\mathcal{T}} = \sigma_{\Gamma(\mathcal{T})}(D_i)$, this implies that
$\theta$ is a unifier of 
$\{ C_1^{\mathcal{T}}\equiv^? D_1^{\mathcal{T}},\ldots,C_n^{\mathcal{T}}\equiv^? D_n^{\mathcal{T}}\}$,
and thus of
$\Gamma = \{ C_1\equiv_{\mathcal{T}}^? D_1,\ldots,C_n\equiv_{\mathcal{T}}^? D_n\}$
modulo $\mathcal{T}$, by Lemma~\ref{expansion:lem}.
\end{proof}

Since the size of $\Gamma(\mathcal{T})$ is basically the same as the size of $\mathcal{T}$,
the size of $\Gamma\cup \Gamma(\mathcal{T})$ is linear in the size of $\Gamma$ and $\mathcal{T}$.
Thus, the above lemma provides us with a 
polynomial-time reduction of $\mathcal{EL}$-unification w.r.t.\ acyclic TBoxes
to $\mathcal{EL}$-unification.

\begin{thm}\label{reduction:thm}
$\mathcal{EL}$-unification w.r.t.\ acyclic TBoxes can be reduced in polynomial time to
$\mathcal{EL}$-unification.
\end{thm}

\section{Equivalence and subsumption in $\mathcal{ EL}$}\label{subs}
In order to characterize equivalence of $\mathcal{ EL}$-concept terms, the
notion  of a reduced $\mathcal{ EL}$-concept term is introduced in \cite{Kues01}.
A given $\mathcal{ EL}$-concept term can be transformed into an equivalent reduced term
by applying the following rules modulo associativity and commutativity of
conjunction:
$$
\begin{array}{l@{\ \ \ \ }l}
C\sqcap \top \rightarrow C & \mbox{for all $\mathcal{ EL}$-concept terms $C$}\\[.3em]
A\sqcap A \rightarrow A    &\mbox{for all concept names $A\in N_\con$}\\[.3em]
\exists\, r.C\sqcap \exists\, r.D \rightarrow \exists\, r.C & \mbox{for all $\mathcal{ EL}$-concept terms $C, D$
                                                               with $C\sqsubseteq D$}
\end{array}
$$
Obviously, these rules are equivalence preserving.
We say that the $\mathcal{ EL}$-concept term $D$ \emph{is reduced} if none of the above
rules is applicable to it (modulo associativity and commutativity of $\sqcap$), and that
$C$ \emph{can be reduced to} $D$ if $D$ can be obtained from $C$ by applying the above rules 
(modulo associativity and commutativity of $\sqcap$).
The $\mathcal{ EL}$-concept term $D$ is a \emph{reduced form} of $C$ if $C$ can be reduced
to $D$ and $D$ is reduced.
The following theorem is an easy consequence of Theorem~6.3.1 on page~181 of \cite{Kues01}.

%\begin{thm}
%Let $C, D$ be reduced $\mathcal{ EL}$-concept terms. Then $C\equiv D$ iff $C$ is identical to $D$
%up to associativity and commutativity of $\sqcap$.
%\end{thm}
%
%As an easy consequence of this theorem, we obtain:

\begin{thm}\label{cor1}
Let $C, D$ be $\mathcal{ EL}$-concept terms, and $\widehat{C}, \widehat{D}$ reduced forms
of $C, D$, respectively. Then $C\equiv D$ iff $\widehat{C}$ is identical to $\widehat{D}$
up to associativity and commutativity of $\sqcap$.
\end{thm}

This theorem can also be used to derive a recursive characterization of
subsumption in $\mathcal{EL}$. In fact,
if $C\sqsubseteq D$, then $C\sqcap D\equiv C$,
and thus $C$ and $C\sqcap D$ have the same reduced form. Thus, during
reduction, all concept names and existential restrictions of $D$ must be 
``eaten up'' by corresponding concept names and existential restrictions
of $C$.

\begin{cor}\label{cor2}
Let
$C = A_1\sqcap\ldots\sqcap A_k\sqcap\exists\, r_1.C_1\sqcap\ldots\sqcap \exists\, r_m.C_m$ and
$D = B_1\sqcap\ldots\sqcap B_\ell\sqcap\exists\, s_1.D_1\sqcap\ldots\sqcap \exists\, s_n.D_n$,
where $A_1,\ldots, A_k, B_1,\ldots,B_\ell$ are concept names.
Then $C\sqsubseteq D$ iff $\{B_1,\ldots, B_\ell\}\subseteq\{A_1,\ldots,A_k\}$ and
for every $j, 1\leq j\leq n$, there exists an $i,1\leq i\leq m$, such that $r_i = s_j$
and $C_i\sqsubseteq D_j$.
\end{cor}

Note that this corollary also covers the cases where some of
the numbers $k, \ell, m, n$ are zero. The empty conjunction should then be read as $\top$.
The following lemma, which is an immediate consequence of this
corollary, will be used in our proof that $\mathcal{ EL}$ has unification type zero.

\begin{lem}\label{lem1}
If $C, D$ are reduced $\mathcal{ EL}$-concept terms such that
$\exists\, r.D\sqsubseteq C$, then $C$ is either $\top$, or of the
form $C = \exists\, r.C_1\sqcap\ldots\sqcap\exists\, r.C_n$ where
$n\geq 1$;
$C_1,\ldots,C_n$ are reduced and pairwise incomparable w.r.t.\ subsumption;
and $D\sqsubseteq C_1, \ldots, D\sqsubseteq C_n$.
Conversely, 
if $C, D$ are $\mathcal{ EL}$-concept terms
such that $C = \exists\, r.C_1\sqcap\ldots\sqcap\exists\, r.C_n$ and
$D\sqsubseteq C_1, \ldots, D\sqsubseteq C_n$, then $\exists\, r.D\sqsubseteq C$.
\end{lem}

%\begin{proof}
%We have $\exists\, r.D\sqsubseteq C$ iff $C\sqcap \exists\, r.D \equiv \exists\, r.D$.
%Since $\exists\, r.D$ is reduced, any reduced form of $C\sqcap\exists\, r.D$ must be
%identical (up to associativity and commutativity of $\sqcap$) to $\exists\, r.D$.
%If $C\neq \top$, then the only rule that can be applied to reduce $C\sqcap\exists\, r.D$
%is the third one. It is easy to see that we can only obtain $\exists\, r.D$ by applying
%this rule if $C$ is of the form $C = \exists\, r.C_1\sqcap\ldots\sqcap\exists\, r.C_n$
%where $D\sqsubseteq C_1, \ldots, D\sqsubseteq C_n$. Since $C$ was assumed to be
%reduced, the terms $C_1,\ldots,C_n$ must also be reduced and pairwise incomparable 
%w.r.t.\ subsumption.
%\qed
%\end{proof}

The following lemma states several other
obvious consequences of Corollary~\ref{cor2}.

\begin{lem}\label{1:new:lem}\hfill
\begin{enumerate}[\em(1)]
\item
  The existential restriction $\exists\, r.C$ is reduced iff $C$ is reduced.
\item
  Let $C_1\sqcap\ldots\sqcap C_n$ be concept names or existential restrictions.
  Then the conjunction $C_1\sqcap\ldots\sqcap C_n$ is reduced iff 
  $C_1, \ldots, C_n$ are reduced and pairwise incomparable w.r.t.\ subsumption.
\item
  Let $C = C_1\sqcap\ldots\sqcap C_m$ and $D = D_1\sqcap\ldots\sqcap D_n$ be
  conjunctions of $\mathcal{EL}$-concept terms. If,
  for all $i, 1\leq i\leq n$, there exists $j, 1\leq j\leq m$, such that
  $C_j\sqsubseteq D_i$, then $C\sqsubseteq D$.
  If $D_1,\ldots, D_n$ are concept names or existential restrictions,
  then the implication in the other direction also holds.
\end{enumerate}
\end{lem}

In the proof of decidability of $\mathcal{ EL}$-unification, we will make use of the fact that
the inverse strict subsumption order is well-founded. 
%To see this, we must have 
%a closer look at the subsumption relation. 

\begin{prop}\label{well:founded1}
There is no infinite sequence 
$C_0, C_1, C_2, C_3, \ldots$ 
of $\mathcal{ EL}$-concept terms 
such that
$C_0\sqsubset C_1\sqsubset C_2 \sqsubset C_3 \sqsubset \cdots$.
\end{prop}

\begin{proof}
We define the \emph{role depth} of an $\mathcal{ EL}$-concept term $C$ as the maximal nesting of
existential restrictions in $C$. Let $n_0$ be the role depth of $C_0$.
Since $C_0\sqsubseteq C_i$ for $i\geq 1$, it is an easy consequence of Corollary~\ref{cor2}
that the role depth of $C_i$ is bounded by $n_0$, and that $C_i$ contains only
concept and role names occurring in $C_0$. In addition, it is known that, for a given %bound
natural number
$n_0$ and finite sets of concept names $N_{\con}$ and role names $N_{\role}$, there are, up to
equivalence, only finitely many $\mathcal{ EL}$-concept terms 
 built using concept names from
$\mathcal  C$ and role names from $\mathcal  R$ and of a role depth bounded by $n_0$
\cite{BaST07}. 
Consequently,
there are indices $i < j$ such that $C_i \equiv C_j$. This contradicts our assumption
that $C_i\sqsubset C_j$.
%\qed
\end{proof}

\section{An $\mathcal{ EL}$-unification problem of type zero}\label{type:zero}

To show that $\mathcal{ EL}$ has unification type zero, we exhibit
an $\mathcal{ EL}$-unification problem that has this type.

\begin{thm}\label{type:zero:thm}
Let $X, Y$ be variables.
The $\mathcal{ EL}$-unification problem 
$\Gamma := \{ X\sqcap\exists\, r.Y\equiv^? \exists\, r.Y\}$ 
has unification type zero.
\end{thm}

\begin{proof}
It is enough to show that any complete set of unifiers for this problem
is redundant, i.e., contains two different unifiers that are comparable
w.r.t.\ the instantiation preorder. Thus, let $M$ be a complete set of
unifiers for $\Gamma$.

First, note that $M$ must contain a unifier that maps $X$ to an $\mathcal{ EL}$-concept
term not equivalent to $\top$ or $\exists\, r.\top$. In fact, consider a substitution $\tau$ such that
$\tau(X) = \exists\, r.A$ and $\tau(Y) = A$. Obviously, $\tau$ is a unifier of
$\Gamma$. Thus, $M$ must contain a unifier $\sigma$
such that $\sigma\INST\tau$. In particular, this means that there is a substitution
$\lambda$ such that $\exists\, r.A = \tau(X) \equiv \lambda(\sigma(X))$. 
Obviously,
$\sigma(X)\equiv \top$ %($\sigma(X)\equiv \exists\, r.\top$)
would imply $\lambda(\sigma(X))\equiv \top$, % ($\lambda(\sigma(X))\equiv \exists\, r.\top$), 
and thus $\exists\, r.A\equiv \top$, % ($\exists\, r.A\equiv \exists\, r.\top$), 
which is, however, not the case.
Similarly, $\sigma(X)\equiv \exists\, r.\top$
would imply $\lambda(\sigma(X))\equiv \exists\, r.\top$,
and thus $\exists\, r.A\equiv \exists\, r.\top$, which is also not the case.

Thus, let $\sigma\in M$ be such that $\sigma(X)\not\equiv\top$ and $\sigma(X)\not\equiv\exists\, r.\top$. 
Without loss of generality, we assume that $C := \sigma(X)$ and $D := \sigma(Y)$ are
reduced. Since $\sigma$ is a unifier of $\Gamma$,
we have $\exists\, r.D\sqsubseteq C$. Consequently, Lemma~\ref{lem1} yields that
$C$ is of the form $C = \exists\, r.C_1\sqcap\ldots\sqcap\exists\, r.C_n$ where
$n\geq 1$, 
$C_1,\ldots,C_n$ are reduced and pairwise incomparable w.r.t.\ subsumption, and
$D\sqsubseteq C_1, \ldots, D\sqsubseteq C_n$.

We use $\sigma$ to construct a new unifier $\widehat{\sigma}$ as follows:
\begin{eqnarray*}
\widehat{\sigma}(X) &:=& \exists\, r.C_1\sqcap\ldots\sqcap\exists\, r.C_n\sqcap\exists\, r.Z\\
\widehat{\sigma}(Y) &:=& D\sqcap Z
\end{eqnarray*}
where $Z$ is a new variable (i.e., one not occurring in $C, D$).
The second part of Lemma~\ref{lem1} implies that $\widehat{\sigma}$ is indeed a unifier of 
$\Gamma$.

Next, we show that $\widehat{\sigma}\INST\sigma$. To this purpose, we
consider the substitution $\lambda$ that maps $Z$ to $C_1$,
and does not change any of the other variables. Then we have
$\lambda(\widehat{\sigma}(X)) = \exists\, r.C_1\sqcap\ldots\sqcap\exists\, r.C_n\sqcap\exists\, r.C_1
\equiv \exists\, r.C_1\sqcap\ldots\sqcap\exists\, r.C_n = \sigma(X)$ and
$\lambda(\widehat{\sigma}(Y)) = D\sqcap C_1 \equiv D = \sigma(Y)$.
Note that the second equivalence holds since we have $D\sqsubseteq C_1$.

Since $M$ is complete, there exists a unifier $\theta\in M$ such that
$\theta\INST\widehat{\sigma}$. Transitivity of the relation $\INST$ thus
yields $\theta\INST\sigma$. Since $\sigma$ and $\theta$ both belong to $M$,
we have completed the proof of the theorem once we have shown that $\sigma\neq\theta$.
Assume to the contrary that $\sigma = \theta$. Then we have $\sigma\INST\widehat{\sigma}$,
and thus there exists a substitution $\mu$ such that $\mu(\sigma(X))\equiv\widehat{\sigma}(X)$,
i.e.,
\begin{equation}\label{eq1}
\exists\, r.\mu(C_1)\sqcap\ldots\sqcap\exists\, r.\mu(C_n)\equiv
\exists\, r.C_1\sqcap\ldots\sqcap\exists\, r.C_n\sqcap\exists\, r.Z.
\end{equation}
Recall that the concept terms $C_1,\ldots,C_n$ are reduced and 
pairwise incomparable w.r.t.\ subsumption. In addition, 
since $\sigma(X) = \exists\, r.C_1\sqcap\ldots\sqcap\exists\, r.C_n$
is reduced and not equivalent to $\exists\, r.\top$, none of the
concept terms $C_1,\ldots,C_n$ can be equivalent to $\top$.
Finally, $Z$ is a concept name that
does not occur in $C_1,\ldots,C_n$. All this implies that
$\exists\, r.C_1\sqcap\ldots\sqcap\exists\, r.C_n\sqcap\exists\, r.Z$ is reduced.
Obviously, any reduced form for $\exists\, r.\mu(C_1)\sqcap\ldots\sqcap\exists\, r.\mu(C_n)$
is a conjunction of at most $n$ existential restrictions. Thus, Theorem~\ref{cor1}
shows that the above equivalence $(\ref{eq1})$ actually cannot hold.

To sum up, we have shown that $M$ contains two distinct unifiers $\sigma, \theta$
such that $\theta\INST\sigma$. Since $M$ was an arbitrary complete set of unifiers
for $\Gamma$, this shows that this unification
problem cannot have a minimal complete set of unifiers.
%\qed
\end{proof}

%Note that, in this proof, the availability of $\top$ in $\mathcal{EL}$ was not needed.\footnote{%
%The only place where we have said anything about $\top$ was when we excluded that
%$\sigma(X)$ is equivalent to $\top$ or $\exists\, r.\top$.
%}
%For this reason, the result still holds if, instead of $\mathcal{EL}$, we consider
%its sublanguage that has only conjunction and existential restriction as
%concept constructor.

\section{The decision problem}\label{decision:problem}

Before we can describe our decision procedure for $\mathcal{EL}$-unification, we
must introduce some notation. An $\mathcal{EL}$-concept term is called an
\emph{atom} iff it is a concept name (i.e., concept constant or concept variable) 
or an existential restriction $\exists\, r.D$.\footnote{%
Note that $\top$ is \emph{not} an atom.
} 
Obviously, any $\mathcal{EL}$-concept term is (equivalent to) a conjunction of atoms,
where the empty conjunction is $\top$.
% %
The set $\At(C)$ of \emph{atoms of an $\mathcal{EL}$-concept term $C$} is defined inductively:
if $C = \top$, then $\At(C) := \emptyset$;
if $C$ is a concept name, then $\At(C) := \{C\}$;
if $C = \exists\, r.D$ then $\At(C) := \{C\}\cup \At(D)$;
if $C = C_1\sqcap C_2$, then $\At(C) := \At(C_1)\cup\At(C_2)$.

Concept names and existential restrictions
$\exists\, r.D$ where $D$ is a concept name or $\top$ are called \emph{flat atoms}.
An $\mathcal{EL}$-concept term is \emph{flat} iff it is a conjunction
of flat atoms (where the empty conjunction is $\top$). 
The $\mathcal{EL}$-unification problem $\Gamma$ is \emph{flat} 
iff it consists of equations between flat $\mathcal{EL}$-concept terms.
%The $\mathcal{EL}$-unification problem $\Gamma$ is \emph{flat} iff
%it only contains equations of the following form:
% %
%\begin{itemize}
%\item
%  $X \equiv^? C$ where $X$ is a variable and $C$ is a non-variable flat atom;
%\item
%  $X_1\sqcap\ldots \sqcap X_m \equiv^? Y_1\sqcap\ldots\sqcap Y_n$
%  where $X_1, \ldots, X_m, Y_1, \ldots, Y_n$ are variables.
%\end{itemize}
 %
By introducing new concept variables and eliminating $\top$, 
any $\mathcal{EL}$-unification problem $\Gamma$
can be transformed in polynomial time into a flat $\mathcal{EL}$-unification problem $\Gamma'$
such that $\Gamma$ is solvable iff $\Gamma'$ is solvable. Thus, we may assume
without loss of generality that our input $\mathcal{EL}$-unification problems are flat.
Given a flat $\mathcal{EL}$-unification problem $\Gamma = \{ C_1\equiv^? D_1, \ldots,
C_n\equiv^? D_n \}$, we call the atoms of $C_1,D_1,\ldots,C_n,D_n$
the \emph{atoms of $\Gamma$}. Atoms of $\Gamma$ that are not variables (i.e., not
elements of $N_v$) are called \emph{non-variable atoms of $\Gamma$}.

The unifier $\sigma$ of $\Gamma$ is called \emph{reduced} %(\emph{ground})
iff, for all concept variables $X$ occurring in $\Gamma$, the $\mathcal{EL}$-concept term
$\sigma(X)$ is reduced. % (does not contain variables). 
It is \emph{ground} iff, for all concept variables $X$ occurring in $\Gamma$, 
the $\mathcal{EL}$-concept term $\sigma(X)$ does not contain variables.
Obviously, $\Gamma$ is solvable
iff it has a reduced ground unifier. 
%The inverse subsumption order induces a
%well-founded order on ground unifiers as follows. 
Given a ground unifier $\sigma$ of $\Gamma$, 
the \emph{atoms of $\sigma$} are the atoms of all the concept terms $\sigma(X)$, 
where $X$ ranges over all variables occurring in $\Gamma$.
 
%Given $\mathcal{EL}$-concept terms
%$C, D$, we define $C \gtris D$ iff $C\sqsubset D$. Proposition~\ref{well:founded1} says that
%the strict order $\gtris$ defined this way is well-founded. It is, of course, not strictly necessary
%to introduce the new symbol $\gtris$ for the inverse subsumption relation since we already
%have the symbol $\sqsubset$ for it. The reason for introducing it nevertheless is that we
%want to use the fact that the inverse subsumption relation is well-founded, and thus any
%solvable unification problem has a \emph{minimal} unifier w.r.t.\ that order.
% %
%We think that,
%if we used $\sqsubset$ instead of $\gtris$, it would be less clear what we mean by the
%smaller term in a comparison $C\sqsubset D$ than in a comparison $C\gtris D$.

\begin{rem}\label{occurrence:remark}
In the following, we consider situations where all occurrences of a given reduced atom $D$
in a reduced concept term $C$ are replaced by a more general concept term, i.e., by 
a concept term $D'$ with $D\gtris D'$. However, when we say \emph{occurrence of $D$ in $C$}, 
we mean occurrence modulo equivalence ($\equiv$) rather than syntactic occurrence. For example,
if $C = \exists\, r.(A\sqcap B)\sqcap\exists\, r.(B\sqcap A)$, $D = \exists\, r.(A\sqcap B)$, and
$D' = \exists\, r.A$, then the term obtained by replacing all occurrences of $D$ in $C$ by
$D'$ should be $\exists\, r.A\sqcap\exists\, r.A$, and not $\exists\, r.A\sqcap\exists\, r.(B\sqcap A)$.
Since $C$ and $D$ are reduced, equivalence is actually the same as being identical up to
associativity and commutativity of $\sqcap$. In particular, this means that any concept term that
(syntactically) occurs in $C$ and is equivalent to the atom $D$ is also an atom, i.e., only atoms
can be replaced by $D'$. In order to make this meaning of occurrence explicit
we will call it \emph{occurrence modulo $\AC$} in the following. We will write
$D_1\eqAC D_2$ to express that the atoms $D_1$ and $D_2$ are identical
up to associativity and commutativity of $\sqcap$. Obviously, $D_1\eqAC D_2$ implies
$D_1\equiv D_2$.
\end{rem}

\begin{lem}\label{monotone:greater:lem}
Let $C, D, D'$ be $\mathcal{EL}$-concept terms such that $D$ is a reduced
atom, $D \gtris D'$, and $C$ is reduced and contains at least one occurrence of $D$
modulo $\AC$. 
If $C'$ is obtained from $C$ by replacing all occurrences of $D$ by $D'$, 
then $C \gtris C'$.
\end{lem}

\begin{proof}
We prove the lemma by induction on the size of $C$.
If $C \eqAC D$, then $C' = D'$, and thus $C \equiv D \gtris D' = C'$, which yields
$C \gtris C'$. 
Thus, assume that $C \not\eqAC D$. 
In this case, $C$ cannot be a concept name since it contains the atom $D$.
If $C = \exists\, r. C_1$, then $D$ occurs in $C_1$ modulo $\AC$. By induction, 
we can assume that $C_1 \gtris C_1'$, where $C_1'$ is obtained from $C_1$ by
replacing all occurrences of $D$ (modulo $\AC$) by $D'$. Thus, we have 
$C = \exists\, r. C_1 \gtris \exists\, r.C_1' = C'$ by Corollary~\ref{cor2}.
Finally, assume that $C = C_1\sqcap\ldots\sqcap C_n$ for $n > 1$ atoms
$C_1,\ldots,C_n$. Since $C$ is reduced, these atoms are incomparable
w.r.t.\ subsumption, and since the atom $D$ occurs in $C$ modulo $\AC$ we can assume without loss of
generality that $D$ occurs in $C_1$ modulo $\AC$. Let $C_1',\ldots,C_n'$ be respectively obtained
from  $C_1,\ldots,C_n$ by replacing every occurrence of $D$ (modulo $\AC$) by $D'$, and
then reducing the concept term obtained this way. By induction, we have $C_1 \gtris C_1'$. 
Assume that $C \not\gtris C'$. Since the concept constructors of $\mathcal{EL}$
are monotone w.r.t.\ subsumption $\sqsubseteq$, we have $C \sqsubseteq C'$, and
thus $C \not\gtris C'$ means that $C\equiv C'$. Consequently, $C = C_1\sqcap\ldots\sqcap C_n$
and the reduced form of $C_1'\sqcap\ldots\sqcap C_n'$ must be equal up to associativity
and commutativity of $\sqcap$. 
If $C_1'\sqcap\ldots\sqcap C_n'$ is not reduced, then its reduced form
is actually a conjunction of $m < n$ atoms, which contradicts $C\equiv C'$.
If $C_1'\sqcap\ldots\sqcap C_n'$ is reduced, then $C_1 \gtris C_1'$ implies that
there is an $i \neq 1$ such that $C_i \equiv C_1'$. However, then $C_i \equiv C_1' \sqsupset C_1$
contradicts the fact that the atoms $C_1,\ldots,C_n$ are incomparable
w.r.t.\ subsumption.
%\qed
\end{proof}

Proposition~\ref{well:founded1} says that the inverse strict subsumption order on concept terms
is well-founded.
%the strict order $\gtris$ defined this way is well-founded.
We use this fact to obtain
a well-founded strict order $\succ$ on ground unifiers.
%. Since $\gtris$ is well-founded,
%its multiset extension $>_{m}$ is also well-founded. 
 %
%Given a ground unifier $\sigma$ of $\Gamma$, we consider the multiset $S(\sigma)$ of
%all $\mathcal{EL}$-concept terms $\sigma(X)$, where $X$ ranges over all 
%concept variables occurring in $\Gamma$.
 %
\begin{defi}
Let $\sigma, \theta$ be ground unifiers of $\Gamma$. We define
\begin{enumerate}[(1)]
\item
$\sigma \succeq \theta$ iff 
$\sigma(X)\sqsubseteq\theta(X)$ holds for all variables $X$ occurring in $\Gamma$.
\item
$\sigma \succ \theta$ iff $\sigma \succeq \theta$ and $\theta \not\succeq \sigma$, i.e., iff
$\sigma(X)\sqsubseteq\theta(X)$ holds for all variables $X$ occurring in $\Gamma$,
and $\sigma(X)\gtris\theta(X)$ holds for at least one variable $X$ occurring in $\Gamma$.
\end{enumerate}
\end{defi}
If $\Gamma$ contains $n$ variables, then
$\succeq$ is the $n$-fold product of the order $\sqsubseteq$ with itself.
Since the strict part $\gtris$ of the inverse subsumption order
$\sqsubseteq$ is well-founded by Proposition~\ref{well:founded1},
the strict part $\succ$ of $\succeq$
is also well-founded \cite{BaNiLong98}.
%$S(\sigma) >_m S(\theta)$. 
The ground unifier $\sigma$ of $\Gamma$ is called
\emph{is-minimal} iff there is no ground unifier $\theta$ of $\Gamma$ such that
$\sigma \succ \theta$. The following proposition is an easy consequence
of the fact that $\succ$ is well-founded.

\begin{prop}\label{minimal:unif:prop}
Let $\Gamma$ be an $\mathcal{EL}$-unification problem. Then $\Gamma$ is solvable
iff it has an is-minimal reduced ground unifier.
\end{prop}

In the following, we show that is-minimal reduced ground unifiers of flat
$\mathcal{EL}$-unification problems satisfy properties that make it easy to
check (with an NP-algorithm) whether such a unifier exists or not.

%\begin{lem}
%Let $\Gamma$ be a flat $\mathcal{EL}$-unification problem and $\gamma$
%a minimal reduced ground unifier of $\Gamma$. If $C$ is an atom of $\gamma$,
%then there is a non-variable atom $D$ of $\Gamma$ such that $C\sqsubseteq \gamma(D)$.
%\end{lem}
%
%
%Next, we show that the subsumption relationship in the statement of the above
%lemma can even be replaced by equivalence.
%% Note: the proof of the first lemma can be seen as a special case of the 
%%       proof of the second lemma (empty conjunction yields top).

\begin{lem}\label{atoms:equiv:lem}
Let $\Gamma$ be a flat $\mathcal{EL}$-unification problem and $\gamma$
an is-minimal reduced ground unifier of $\Gamma$. 
If $C$ is an atom of $\gamma$,
then there is a non-variable atom $D$ of $\Gamma$ such that $C\equiv \gamma(D)$.
%If $\exists\, r.C$ is an atom of $\gamma$,
%then there is an atom $\exists\, r. D$ of $\Gamma$ such that $\exists\, r.C\equiv \exists\, r.\gamma(D)$.
\end{lem}

The main idea underlying the proof of this crucial lemma is that an atom $C$ of a unifier
$\sigma$ that violates the condition of the lemma (i.e., that is not of the form
$C\equiv \gamma(D)$ for a non-variable atom $D$ of $\Gamma$) can be replaced by a concept term $\widehat{D}$
such that $C\gtris \widehat{D}$, which yields a unifier
of $\Gamma$ that is smaller than $\sigma$ w.r.t.\ $\succ$.

Before proving the lemma formally, let us illustrate this idea by two examples.

\begin{exa}
First, consider the unification problem 
$$
\Gamma_1 := \{ \exists\, r.X\sqcap\exists\, r.A\equiv^? \exists\, r.X\}.
$$
The substitution $\sigma_1 := \{X\mapsto A\sqcap B\}$ is a unifier of $\Gamma_1$ that
does not satisfy the condition of Lemma~\ref{atoms:equiv:lem}. In fact, $B$ is an
atom of $\sigma_1$, but none of the non-variable atoms $D$ of $\Gamma_1$ (which are
$A$, $\exists\, r.A$, and $\exists\, r.X$) satisfy $B\equiv \sigma_1(D)$.
The unifier $\sigma_1$ is not is-minimal since $\gamma_1 := \{X\mapsto A\}$,
which can be obtained from $\sigma_1$ by replacing the offending atom $B$ with $\top$,
is a unifier of $\Gamma_1$ that is smaller than $\sigma_1$ w.r.t.\ $\succ$. The unifier
$\gamma_1$ is is-minimal, and it clearly
satisfies the condition of Lemma~\ref{atoms:equiv:lem}.

Second, consider the unification problem
$$
\Gamma_2 := \{ X \sqcap \exists\, r.A\sqcap\exists\, r.B\equiv^? X\}.
$$
The substitution $\sigma_2 := \{X\mapsto \exists\, r.(A\sqcap B)\}$ is a unifier of $\Gamma_2$ that
does not satisfy the condition of Lemma~\ref{atoms:equiv:lem}. In fact, $\exists\, r.(A\sqcap B)$ is an
atom of $\sigma_2$, but none of the non-variable atoms $D$ of $\Gamma_2$ (which are
$A$, $B$, $\exists\, r.A$, and $\exists\, r.B$) satisfy $\exists\, r.(A\sqcap B)\equiv \sigma_2(D)$.
The unifier $\sigma_2$ is not is-minimal since $\gamma_2 := \{X\mapsto \exists\, r.A\sqcap\exists\, r.B\}$,
which can be obtained from $\sigma_2$ by replacing the offending atom $\exists\, r.(A\sqcap B)$ with 
$\exists\, r.A\sqcap\exists\, r.B$, is a unifier of $\Gamma_2$ that is smaller than $\sigma_2$ w.r.t.\ $\succ$.
The unifier
$\gamma_2$ is is-minimal, and it clearly
satisfies the condition of Lemma~\ref{atoms:equiv:lem}.
\end{exa}

%\begin{proof}
\subsubsection*{Proof of Lemma~\ref{atoms:equiv:lem}}
Assume that $\gamma$ is an is-minimal reduced ground unifier of $\Gamma$.
Since $\gamma$ is reduced, all atoms of $\gamma$ are reduced.
In particular, this implies that $C$ is reduced, and 
since $\gamma$ is ground, we know that $C$ is either a concept constant or an existential
restriction.

\emph{First}, assume that \emph{$C$ is of the form $A$ for a concept constant $A$}, but there is
no non-variable atom $D$ of $\Gamma$ such that $A\equiv \gamma(D)$. This simply
means that $A$ does not appear in $\Gamma$. Let $\gamma'$ be the substitution
obtained from $\gamma$ by replacing every occurrence of $A$ by $\top$. Since
equivalence in $\mathcal{EL}$ is preserved under replacing concept names by
$\top$, and since $A$ does not appear in $\Gamma$, it is easy to see that $\gamma'$
is also a unifier of $\Gamma$. However, since $\gamma \succ \gamma'$,
this contradicts our assumption that $\gamma$ is is-minimal. % unifier of $\Gamma$.

\emph{Second}, assume that 
\emph{$C$ is an existential restriction of the form $\exists\, r.C_1$}, but there is
no non-variable atom $D$ of $\Gamma$ such that $C\equiv \gamma(D)$. 
We assume that $C$ is maximal (w.r.t.\ subsumption) with this property,
i.e., for every atom $C'$ of $\gamma$ with $C\sqsubset C'$, there is a
non-variable atom $D'$ of $\Gamma$ such that $C'\equiv \gamma(D')$.
Let $D_1,\ldots, D_\ell$ be all the  non-variable atoms of $\Gamma$ with $C\sqsubseteq\gamma(D_i)$
($i = 1,\ldots,\ell$). By our assumptions on $C$, we actually have
$C\sqsubset\gamma(D_i)$ and, by Lemma~\ref{lem1}, the atom $D_i$ is also an existential restriction
$D_i = \exists\, r.D_i'$ ($i = 1,\ldots,\ell$). 
We consider the conjunction
$$
\widehat{D} := \gamma(D_1)\sqcap\ldots\sqcap\gamma(D_\ell),
$$ 
 which is $\top$ in case $\ell = 0$.

\begin{defi}
Given an $\mathcal{EL}$-concept term $F$, the concept term $F\CDh$ is obtained from 
$F$ by replacing every occurrence of $C$ (modulo $\AC$) by $\widehat{D}$. 
The substitution $\gamma\CDh$ is obtained from $\gamma$ by replacing every occurrence
of $C$ (modulo $\AC$) by $\widehat{D}$, i.e., $\gamma\CDh(X) := \gamma(X)\CDh$ for
all variables $X$.
\end{defi}

We will show in the following that $\gamma\CDh$ is a unifier of $\Gamma$ that is smaller than $\gamma$
w.r.t.\ $\succ$. This will then again contradict our assumption that $\gamma$ is is-minimal.

%First, we show that 
\begin{lem}
$\gamma \succ \gamma\CDh$.
\end{lem}
\begin{proof}
Obviously, $\widehat{D}$
%$\widehat{D} := \gamma(D_1)\sqcap\ldots\sqcap\gamma(D_n)$ 
subsumes $C$.
We claim that this subsumption relationship is actually strict.
In fact, if $\ell = 0$, then $\widehat{D} = \top$, and since $C$ is an atom, it is
not equivalent to $\top$. If $\ell \geq 1$, then 
$C = \exists\, r.C_1 \sqsupseteq  
\exists\, r.\gamma(D_1')\sqcap\ldots\sqcap\exists\, r.\gamma(D_\ell')$ would imply
(by Corollary~\ref{cor2}) that there is an $i, 1\leq i\leq \ell$, with $C_1\sqsupseteq \gamma(D_i')$.
However, this would yield $C = \exists\, r.C_1 \sqsupseteq \exists\, r.\gamma(D_i') = \gamma(D_i)$,
which contradicts the fact that $C\sqsubset\gamma(D_i)$. Thus, we have shown that
$C \sqsubset \widehat{D}$.
% i.e., every atom of $\gamma$ that is equivalent to $C$ is replaced by $\widehat{D}$. 
Lemma~\ref{monotone:greater:lem} implies that $\gamma \succ \gamma'$.
\end{proof}

To complete the proof of Lemma~\ref{atoms:equiv:lem}, it remains to show the next lemma.

\begin{lem}
$\gamma\CDh$ is a unifier of $\Gamma$.
\end{lem}
 %
%\begin{proof}
\proof
Consider an equation in $\Gamma$ of the form
$L_1\sqcap\ldots \sqcap L_m \equiv^? R_1\sqcap\ldots\sqcap R_n$
where $L_1, \ldots, L_m$ and $R_1, \ldots, R_n$ are flat atoms, and define
$L := \gamma(L_1\sqcap\ldots \sqcap L_m)$ and $R := \gamma(R_1\sqcap\ldots\sqcap R_n)$.
We know that $L, R$ are conjunctions of atoms of the form
$L = A_1\sqcap\ldots\sqcap A_\mu$ and $R = B_1\sqcap\ldots\sqcap B_\nu$,
where each conjunct $A_1,\ldots,A_\mu, B_1,\ldots,B_\nu$
is a reduced ground atom that is either an atom of $\gamma$ or 
equal to $\gamma(E)$ for a non-variable atom $E$ of $\Gamma$. 
Since $\gamma$ is a unifier of $\Gamma$, we have $L\equiv R$.

\begin{enumerate}[(1)]
\item\label{part:one}
Since $C$ is an atom,
we obviously have $L\CDh = A_1\CDh\sqcap\ldots\sqcap A_\mu\CDh$ and $R\CDh = B_1\CDh\sqcap\ldots\sqcap B_\nu\CDh$.
Now, we show that 
$L\CDh = \gamma\CDh(L_1\sqcap\ldots\sqcap L_m)$ and $R\CDh = \gamma\CDh(R_1\sqcap\ldots\sqcap R_n)$.
We concentrate on proving the first identity since the second one can be shown analogously.
To show the first identity, it is enough to prove that 
$\gamma(L_j)\CDh = \gamma\CDh(L_j)$ holds for all $j, 1\leq j\leq m$. 
\begin{enumerate}[(a)]
\item\label{case:a}
If $L_j$ is a variable $X$, then $\gamma\CDh(X) = \gamma(X)\CDh$ holds by the definition
of $\gamma\CDh$.
\item\label{case:b}
If $L_j$ is a concept constant $A$, then $A\CDh = A$ since $C$ is an existential restriction.
Thus, we have $\gamma\CDh(A) = A = A\CDh = \gamma(A)\CDh$.
\item\label{case:c}
Otherwise, $L_j$ is an existential restriction $\exists\, r_j.L_j'$. By our assumption
on $C$, we have $C \not\equiv \gamma(L_j)$, and thus 
$\gamma(L_j)\CDh = \exists\, r_j.\left(\gamma(L_j)\CDh\right)$. In addition, we have
$\gamma\CDh(L_j) = \exists\, r_j.\gamma\CDh(L_j')$. Thus, it is enough to show
$\gamma(L_j')\CDh = \gamma\CDh(L_j')$. Since $L_j$ is a flat atom, we know that
$L_j'$ is either a concept constant, the top-concept $\top$, or a concept variable. In the first
to cases, we can show $\gamma(L_j')\CDh = \gamma\CDh(L_j')$ as in (\ref{case:b}), and in the third
case we can show this identity as in (\ref{case:a}).
\end{enumerate}

\item
Because of (\ref{part:one}), if we can prove that $L\CDh\equiv R\CDh$,
then we have shown that $\gamma\CDh$ solves the equation
$L_1\sqcap\ldots \sqcap L_m \equiv^? R_1\sqcap\ldots\sqcap R_n$.

Without loss of generality, we concentrate on showing that $L\CDh\sqsubseteq R\CDh$.
Since $L\CDh = A_1\CDh\sqcap\ldots\sqcap A_\mu\CDh$ and $R\CDh = B_1\CDh\sqcap\ldots\sqcap B_\nu\CDh$,
it is thus sufficient to show that, for every $i, 1\leq i\leq \nu$,
there exists a $j, 1\leq j\leq \mu$, such that $A_j\CDh\sqsubseteq B_i\CDh$
(see (3) of Lemma~\ref{1:new:lem}). 
Since $L = A_1\sqcap\ldots\sqcap A_\mu \sqsubseteq B_1\sqcap\ldots\sqcap B_\nu = R$ and 
$A_1,\ldots,A_\mu, B_1,\ldots, B_\nu$ are atoms, we actually know that, for every $i, 1\leq i\leq \nu$,
there exists a $j, 1\leq j\leq \mu$, such that $A_j\sqsubseteq B_i$.
Thus, it is sufficient to show that $A_j\sqsubseteq B_i$ implies
$A_j\CDh\sqsubseteq B_i\CDh$. This is an easy consequence of the next lemma
since $A_i, B_j$ satisfy the conditions of this lemma.\hfill\qed
\end{enumerate}
%\end{proof} 

%Without loss of generality, we concentrate on the reductions for the left-hand side.
%Since the atoms $A_1,\ldots,A_\mu$ are reduced, the reductions used to transform
%$L = A_1\sqcap\ldots\sqcap A_\mu$ into $J$ are all of the form $A_i\sqcap A_j\rightarrow A_i$
%where $A_i\sqsubseteq A_j$. Since $L' = A_1'\sqcap\ldots\sqcap A_\mu'$ (where $ A_1',\ldots, A_\mu'$
%are obtained from $A_1,\ldots,A_\mu$ by replacing every occurrence of $C$ by $\widehat{D}$),
%it is enough to show that $A_i\sqsubseteq A_j$ implies $A_i'\sqsubseteq A_j'$.
%In fact, by (3) of Lemma~\ref{1:new:lem} this then implies that $A_i'\sqcap A_j'$
%can be reduced to $A_i'$. 

\begin{lem}\label{hilf:2}
Let $A, B$ be reduced ground atoms such that $B$ is an
atom of $\gamma$ or of the form $\gamma(D)$ for a non-variable atom $D$ of $\Gamma$. 
%If $A',B'$ are obtained from $A, B$ by replacing every occurrence of $C$ by $\widehat{D}$, 
If $A\sqsubseteq B$, then
$A\CDh\sqsubseteq B\CDh$.
\end{lem}
\proof
We show $A\CDh\sqsubseteq B\CDh$ by induction on the size of $A$. 
\begin{enumerate}[(1)]
\item
First, assume that
$A \eqAC C$, which implies that $A\CDh = \widehat{D} = \gamma(D_1)\sqcap\ldots\sqcap\gamma(D_n)$. 
\begin{enumerate}[(a)]
\item
If $B$ is of the form $B \equiv \gamma(D)$ for a non-variable atom $D$ of $\Gamma$,
then there is an  $h, 1\leq h\leq n$, such that $D = D_h$, which shows that
$A\CDh\sqsubseteq B$. 
Since $C\sqsubseteq \widehat{D}$ and the constructors of $\mathcal{EL}$ are monotone w.r.t.\ subsumption,
we also have $B\sqsubseteq B\CDh$, and thus $A\CDh\sqsubseteq B\CDh$. 
\item
Assume that $B$ is an atom of $\gamma$. If $B \eqAC C$, then $B\CDh = \widehat{D}$, and
thus $A\CDh = B\CDh$, which implies $A\CDh\sqsubseteq B\CDh$. Otherwise, since $C, B$ are reduced atoms,
$B \not\eqAC C$ implies $B \not\equiv C$. Together with $C\equiv A\sqsubseteq B$, this shows
that $C \sqsubset B$. Thus,
the maximality of $C$ implies that there is a non-variable atom $D$ of $\Gamma$ such that
$B \equiv \gamma(D)$. Thus, we are actually in case (a), which yields $A\CDh\sqsubseteq B\CDh$.
\end{enumerate}
\item
Now, assume that $A \not\eqAC C$. If there is no occurrence (modulo $\AC$) of $C$ in $A$,
then we have $A\CDh = A \sqsubseteq B \sqsubseteq B\CDh$. 

Otherwise, $A$ is of the form $A = \exists\, s.E$ and 
$C$ occurs in $E$ (modulo $\AC$). Obviously, $A\sqsubseteq B$ then implies that
$B$ is of the form $B = \exists\, s.F$ with $E\sqsubseteq F$. The concept terms $E, F$ are
conjunctions of reduced ground atoms, i.e.,
$E = E_1\sqcap\ldots\sqcap E_\kappa$ and $F = F_1\sqcap\ldots\sqcap F_\lambda$
where $E_1,\ldots,E_\kappa,F_1,\ldots,F_\lambda$ are reduced ground atoms. % that are atoms of $\gamma$. 
By Corollary~\ref{cor2},
for every $h, 1\leq h\leq \lambda$, there exists $k, 1\leq k\leq \kappa$ such that
$E_k\sqsubseteq F_h$. 

In order to be able to assume, by induction, that $E_k\sqsubseteq F_h$ implies
$E_k\CDh\sqsubseteq F_h\CDh$, we must show that the conditions in the statement
of the lemma hold for the concept terms $E_k, F_h$, where $E_k$ plays the r\^ole of $A$ and
$F_h$ plays the r\^ole of $B$. 
Since we already know that
$E_1,\ldots,E_\kappa,F_1,\ldots,F_\lambda$ are reduced ground atoms, it is sufficient to show that each
of the atoms
$F_1,\ldots,F_\lambda$ is an atom of $\gamma$ or of the form 
$\gamma(D)$ for a non-variable atom $D$ of $\Gamma$. We know that $B = \exists\, s.(F_1\sqcap\ldots\sqcap F_\lambda)$
is an atom of $\gamma$ or an instance (w.r.t.\ $\gamma$) of a non-variable atom of $\Gamma$.
In the first case, the atoms $F_1,\ldots,F_\lambda$ are clearly also atoms of $\gamma$.
In the second case, $B = \gamma(D')$ for a non-variable atom $D'$ of $\Gamma$. If $D'$ is a ground
atom, then $F_1,\ldots,F_\lambda$ are also ground atoms that are atoms of $\Gamma$, and thus they
are instances (w.r.t.\ $\gamma$) of non-variable atoms of $\Gamma$. Otherwise, since $\Gamma$ is
flat, $D'$ is of the form $\exists\, s.X$ for a variable $X$ and $\gamma(X) =  F_1\sqcap\ldots\sqcap F_\lambda$.
In this case, $F_1,\ldots,F_\lambda$ are clearly atoms of $\gamma$.

Thus, we can assume by induction:
$$
(*)\ \ 
\mbox{for every $h, 1\leq h\leq \lambda$, there exists 
$k, 1\leq k\leq \kappa$ such that $E_k\CDh\sqsubseteq F_h\CDh$} 
$$
It remains to show that this implies $A\CDh \sqsubseteq B\CDh$.
\begin{enumerate}[(a)]
\item
 If $B\not\eqAC C$, then $A\CDh = \exists\, s.(E_1\CDh\sqcap\ldots\sqcap E_\kappa\CDh)$
 and $B\CDh = \exists\, s.(F_1\CDh\sqcap\ldots\sqcap F_\lambda\CDh)$, and thus property $(*)$
 yields $A\CDh\sqsubseteq B\CDh$.
\item
 Assume that $B\eqAC C$. In this case, $C$ cannot occur (modulo $\AC$) in
 any of the concept terms $F_1,\ldots, F_h$, which implies that
 $B = \exists\, s.(F_1\sqcap\ldots\sqcap F_\lambda) = \exists\, s.(F_1\CDh\sqcap\ldots\sqcap F_\lambda\CDh)$. 
 Since we have $A\CDh = \exists\, s.(E_1\CDh\sqcap\ldots\sqcap E_\kappa\CDh)$, property $(*)$
 yields $A\CDh\sqsubseteq B$.
 Since %the constructors of $\mathcal{EL}$ are monotone w.r.t.\ subsumption,
 we also have $B\sqsubseteq B\CDh$, this yields $A\CDh\sqsubseteq B\CDh$.
\end{enumerate}
\end{enumerate}
 % 
%Assume that an occurrence of $C$ in $A_i$ is actually needed to have the
%subsumption $A_i\sqsubseteq A_j$. Then there is an existential restriction
%$C'$ in $A_j$ such that $C \sqsubseteq C'$. If $C = C'$, then both are replaced
%by $\widehat{D}$, and thus this replacement is harmless. Otherwise,
%$C\sqsubset C'$. 
% %
%If $C'$ is an atom of $\gamma$, then maximality of $C$ yields
%that there is a non-variable atom $D'$ of $\Gamma$ such that $C'\equiv \gamma(D')$.
%If $C'$ is not an atom of $\gamma$, then it must be equal to $A_j$, which in turn must
%be equal to $\gamma(D')$ for a non-variable atom $D'$ of $\Gamma$.
 %
%In both cases, $C\sqsubset C' \equiv \gamma(D')$ implies that there is an $i, 1\leq i\leq n$,
%such that $D' = D_i$. Thus, $C'$ is actually one of the conjuncts of
%$\widehat{D}$, which again shows that replacing $C$ by $\widehat{D}$
%is harmless. 
 %
Thus, we have shown in all cases that $A\CDh\sqsubseteq B\CDh$, which
completes the proof of Lemma~\ref{hilf:2}.
\qed
%\end{proof}

Overall, we have thus completed the proof of Lemma~\ref{atoms:equiv:lem}. 
The next proposition is an easy consequence of this lemma.

\begin{prop}\label{main:prop}
Let $\Gamma$ be a flat $\mathcal{EL}$-unification problem and $\gamma$
an is-minimal reduced ground unifier of $\Gamma$. If $X$ is a concept variable
occurring in $\Gamma$, then $\gamma(X) \equiv \top$ or there are non-variable atoms
$D_1,\ldots, D_n$ ($n\geq 1$) of $\Gamma$ such that
$\gamma(X) \equiv \gamma(D_1)\sqcap\ldots\sqcap\gamma(D_n)$.
\end{prop}

\begin{proof}
If $\gamma(X)\not\equiv \top$, then it is a non-empty conjunction of atoms,
i.e., there are atoms $C_1,\ldots, C_n$ ($n\geq 1$) such that 
$\gamma(X) = C_1\sqcap\ldots\sqcap C_n$. Then $C_1,\ldots, C_n$ are atoms of
$\gamma$, and thus Lemma~\ref{atoms:equiv:lem} yields non-variable atoms
$D_1,\ldots,D_n$ of $\Gamma$ such that $C_i\equiv \gamma(D_i)$ for $i = 1,\ldots n$.
Consequently, $\gamma(X) \equiv \gamma(D_1)\sqcap\ldots\sqcap\gamma(D_n)$.
%\qed
\end{proof}

This proposition suggests the following non-deterministic algorithm for deciding solvability
of a given flat $\mathcal{EL}$-unification problem.
\begin{algo}\label{algo:1}
Let $\Gamma$ be a flat $\mathcal{EL}$-unification problem.
\begin{enumerate}[(1)]
\item
  For every variable $X$ occurring in $\Gamma$, guess a finite, possibly empty,
  set $S_X$ of non-variable atoms of $\Gamma$.
\item
  We say that the variable $X$ \emph{directly depends on} the variable $Y$
  if $Y$ occurs in an atom of $S_X$. Let \emph{depends on} be the transitive
  closure of \emph{directly depends on}. If there is a variable that depends on
  itself, then the algorithm returns ``fail.'' Otherwise, there
  exists a strict linear order $>$ on the variables occurring in $\Gamma$ such
  that $X > Y$ if $X$ depends on $Y$.
\item
  We define the substitution $\sigma$ along the linear order $>$:
  \begin{enumerate}[$\bullet$]
  \item
    If $X$ is the least variable w.r.t.\ $>$, then $S_X$ does not contain
    any variables. We define $\sigma(X)$ to be the conjunction of the
    elements of $S_X$, where the empty conjunction is $\top$.
  \item
    Assume that $\sigma(Y)$ is defined for all variables $Y < X$. Then
    $S_X$ only contains variables $Y$ for which $\sigma(Y)$ is already
    defined. If $S_X$ is empty, then we define $\sigma(X) := \top$.
    Otherwise, let $S_X = \{D_1,\ldots,D_n\}$. We define
    $\sigma(X) := \sigma(D_1)\sqcap\ldots\sqcap\sigma(D_n)$.
  \end{enumerate}
\item
  Test whether the substitution $\sigma$ computed in the previous step
  is a unifier of $\Gamma$. If this is the case, then return $\sigma$;
  otherwise, return ``fail.''
\end{enumerate}
\end{algo}
This algorithm is trivially \emph{sound} since it only returns substitutions that
are unifiers of $\Gamma$. In addition, it obviously always terminates. Thus, to 
show correctness of our algorithm, it is sufficient to show that it is complete.

\begin{lem}[Completeness]
If $\Gamma$ is solvable, then there is a way of guessing in Step~1
subsets $S_X$ of the non-variable atoms of $\Gamma$ such that the \emph{depends on}
relation determined in Step~2 is acyclic and the substitution $\sigma$ 
computed in Step~3 is a unifier of $\Gamma$.
\end{lem}

\begin{proof}
If $\Gamma$ is solvable, then it has an is-minimal reduced ground unifier $\gamma$.
By Proposition~\ref{main:prop}, for every variable $X$ occurring in
$\Gamma$ we have $\gamma(X) \equiv \top$ or there are non-variable atoms
$D_1,\ldots, D_n$ ($n\geq 1$) of $\Gamma$ such that
$\gamma(X) \equiv \gamma(D_1)\sqcap\ldots\sqcap\gamma(D_n)$.
If $\gamma(X) \equiv \top$, then we define $S_X := \emptyset$. Otherwise,
we define $S_X := \{D_1,\ldots, D_n\}$.

We show that the relation \emph{depends on} induced by these sets $S_X$ is
acyclic, i.e., there is no variable $X$ such that $X$ depends on itself.
If $X$ directly depends on $Y$, then $Y$ occurs in an element of $S_X$. Since
$S_X$ consists of non-variable atoms of the flat unification 
problem $\Gamma$, this means that there is a role name
$r$ such that $\exists\, r.Y \in S_X$. Consequently, we have 
$\gamma(X)\sqsubseteq\exists\, r.\gamma(Y)$. Thus, if $X$ depends on $X$, then
there are $k\geq 1$ role  names $r_1,\ldots,r_k$ such that
$\gamma(X)\sqsubseteq \exists\, r_1.\cdots\exists\, r_k.\gamma(X)$.
This is clearly not possible since $\gamma(X)$ cannot be subsumed by
an $\mathcal{EL}$-concept term whose role depth is larger than the role depth
of $\gamma(X)$.

To show that the substitution $\sigma$ induced by the sets $S_X$ is a unifier
of $\Gamma$, we prove that $\sigma$ is equivalent to $\gamma$, i.e.,
$\sigma(X)\equiv\gamma(X)$ holds for all variables $X$ occurring in $\Gamma$.
The substitution $\sigma$ is defined along the linear order $>$. 
If $X$ is the least variable w.r.t.\ $>$, then the elements of $S_X$ do not contain
any variables. If $S_X$ is empty, then $\sigma(X) = \top \equiv \gamma(X)$.
Otherwise, let $S_X = \{D_1,\ldots,D_n\}$. Since the atoms $D_i$ do not contain
variables, we have $D_i = \gamma(D_i)$.
Thus, the definitions of $S_X$ and of $\sigma$ yield
$\sigma(X) = D_1\sqcap\ldots\sqcap D_n = \gamma(D_1)\sqcap\ldots\sqcap\gamma(D_n) \equiv
\gamma(X)$.

Assume that $\sigma(Y)\equiv\gamma(Y)$ holds for all variables $Y < X$. 
If $S_X = \emptyset$, then we have again $\sigma(X) = \top \equiv \gamma(X)$.
Otherwise, let $S_X = \{D_1,\ldots,D_n\}$. Since the atoms $D_i$ contain
only variables that are smaller than $X$, we have $\sigma(D_i) \equiv \gamma(D_i)$
by induction. 
Thus, the definitions of $S_X$ and of $\sigma$ yield
$\sigma(X) = \sigma(D_1)\sqcap\ldots\sqcap \sigma(D_n) \equiv 
\gamma(D_1)\sqcap\ldots\sqcap\gamma(D_n) \equiv \gamma(X)$.
%\qed
\end{proof}

Note that our proof of completeness actually shows that, up to equivalence, the algorithm 
returns all is-minimal reduced ground unifiers of $\Gamma$.

\begin{thm}\label{NP:complete:thm}
$\mathcal{EL}$-unification is NP-complete.
\end{thm}

\begin{proof}
NP-hardness follows from the fact that $\mathcal{EL}$-matching is NP-complete
\cite{Kues01}.\footnote{% 
The NP-hardness proof in \cite{Kues01} is by reduction of SAT. This reduction
employs two concept constants and four role names. However, the roles are mainly 
used to encode several (matching) equations into a single one. When using
a set of equations rather than a single equation, one role name is sufficient.
}
To show that the problem can be decided by a non-deterministic
polynomial-time algorithm, we analyze the complexity of our algorithm.
Obviously, guessing the sets $S_X$ (Step~1) can be done within NP. 
Computing the \emph{depends on} relation and checking it for acyclicity
(Step~2) is clearly polynomial.

Steps~3 and~4 are more problematic. In fact, since a variable may occur
in different atoms of $\Gamma$, the substitution $\sigma$ computed in 
Step~3 may be of exponential size. This is actually the same reason that
makes a naive algorithm for syntactic unification compute an exponentially
large most general unifier \cite{BaSn01}. As in the case of syntactic
unification, the solution to this problem is basically structure sharing.
Instead of computing the substitution $\sigma$ explicitly, we view its definition
as an acyclic TBox. To be more precise, for every concept variable $X$
occurring in $\Gamma$, the TBox $\mathcal{T}_\sigma$ contains
the concept definition $X \doteq \top$ if $S_X = \emptyset$ and
$X\doteq D_1\sqcap\ldots\sqcap D_n$ if  $S_X = \{D_1,\ldots,D_n\}$ ($n\geq 1$).
Instead of computing $\sigma$ in Step~3, we compute $\mathcal{T}_\sigma$.
Because of the acyclicity test in Step~2, we know that $\mathcal{T}_\sigma$
is an acyclic TBox.
The size of $\mathcal{T}_\sigma$ is obviously polynomial in the size of $\Gamma$,
and thus this modified Step~3 is polynomial.

It is easy to see that applying the substitution $\sigma$ to a concept term
$C$ is the same as expanding $C$ w.r.t.\ the TBox $\mathcal{T}_\sigma$, i.e.,
$\sigma(C) = C^{\mathcal{T}_\sigma}$.
This implies that, for every equation $C\equiv^? D$ in
$\Gamma$, we have $C\equiv_{\mathcal{T}_\sigma} D$ iff $\sigma(C)\equiv\sigma(D)$.
Thus, testing in Step~4 whether $\sigma$ is a unifier of $\Gamma$ can be reduced to
testing whether $C\equiv_{\mathcal{T}_\sigma} D$ holds for every
equation $C\equiv^? D$ in $\Gamma$. Since subsumption (and thus equivalence) in $\mathcal{EL}$
w.r.t.\ acyclic TBoxes can be decided in polynomial time \cite{Baad03e},
%\footnote{%
%Of course, the polynomial-time subsumption algorithm does not expand the TBox.
%}
this completes the proof of the theorem.
%\qed
\end{proof}

In Subsection~\ref{unif:mod:TBox}, we have shown
 that there exists a polynomial-time 
reduction of unification modulo an acyclic TBox to unification without a TBox. Thus,
Theorem~\ref{NP:complete:thm} also yields the exact complexity for $\mathcal{EL}$-unification 
w.r.t.\ acyclic TBoxes.

\begin{cor}
$\mathcal{EL}$-unification w.r.t.\ acyclic TBoxes is NP-complete.
\end{cor}

\begin{proof}
The problem is in NP since Theorem~\ref{reduction:thm} states that there is
a polynomial-time reduction of $\mathcal{EL}$-unification w.r.t.\ acyclic TBoxes
to $\mathcal{EL}$-unification, and we have just shown that $\mathcal{EL}$-unification
is in NP.

NP-hardness for $\mathcal{EL}$-unification w.r.t.\ acyclic TBoxes follows from 
NP-hardness of $\mathcal{EL}$-unification since $\mathcal{EL}$-unification can be viewed
as the special case of $\mathcal{EL}$-unification w.r.t.\ acyclic TBoxes where the TBox
is empty.
\end{proof}

\section{A goal-oriented algorithm}\label{goal:oriented:sect}

The NP-algorithm  introduced in the previous section is a typical ``guess and then test''
NP-algorithm, and thus it is unlikely that a direct implementation of this algorithm
will perform well in practice. Here, we introduce a more goal-oriented
unification algorithm for $\mathcal{EL}$, in which non-deterministic decisions are only made
if they are triggered by ``unsolved parts'' of the unification problem.

As in the previous section, we assume without loss of generality
that our input unification problem $\Gamma_0$ is flat.
For a given flat equation $C\equiv^? D$, the concept terms $C, D$ are thus
conjunctions of flat atoms. We will often view such an equation as consisting of four sets:
the left-hand side $C$ is given by the set of variables occurring in the top-level conjunction of $C$, 
together with the set of non-variable atoms occurring in this top-level conjunction;
the right-hand side $D$ is given by the set of variables occurring in the top-level conjunction of $D$, 
together with the set of non-variable atoms occurring in this top-level conjunction. 
To be more precise, let $e$ denote the equation
$C\equiv^? D$, where $C = X_1\sqcap\ldots\sqcap X_m\sqcap A_1\sqcap\ldots\sqcap A_k$ and
$D = Y_1\sqcap\ldots\sqcap Y_n\sqcap B_1\sqcap\ldots\sqcap B_\ell$ for concept variables
$X_1,\ldots,X_m,Y_1,\ldots,Y_n$ and non-variable atoms $A_1, \ldots, A_k, B_1, \ldots B_\ell$.
Then we define
$$
\begin{array}{ll}
\LV(e) := \{X_1,\ldots,X_m\}, & \RV(e) := \{Y_1,\ldots,Y_n\},\\
\LA(e) := \{A_1,\ldots,A_k\}, & \RA(e) := \{B_1, \ldots B_\ell\}.
\end{array}
$$
Obviously, the equation $e : C\equiv^? D$ is uniquely determined (up to associativity, commutativity, 
and idempotency of conjunction) by the four sets $\LV(e), \LA(e), \RV(e), \RA(e)$.
Instead of viewing an equation $e$ as being given by a pair of concept terms, we can thus also
view it as being given by these four sets. In the following, it will often be convenient to employ this
representation of equations.
If, with this point of view, we say that we \emph{add an atom} to the set $\LA(e)$  or $\RA(e)$,
then this means, for the other point of view, that we conjoin this atom to the top-level conjunction of
the left-hand side or right-hand side of the equation. In addition, if we say
that the equation $e$ %(the left-hand side of $e$, the right-hand side of $e$)
\emph{contains} the variable $X$, then we mean that $X\in \LV(e)\cup\RV(e)$.
Similarly, if we say that the left-hand side of $e$ contains $X$, then we mean that
$X\in \LV(e)$, and if we say that the right-hand side of $e$ contains $X$, 
then we mean that $X\in \RV(e)$).\footnote{
Note that occurrences of $X$ inside non-variable atoms $\exists\, r.X\in \LA(e)\cup\RA(e)$
are not taken into consideration here.
}

In addition to the unification problem itself,
the algorithm also maintains, for every variable $X$ occurring in the input problem $\Gamma_0$,
a set $S_X$ of non-variable atoms of $\Gamma_0$. Initially, all the sets $S_X$ are empty.
We call the set $S_X$ the current assignment for $X$, and the collection of all these sets the
\emph{current assignment}. Throughout the run of our goal-oriented algorithm, we will ensure
that the current assignment is \emph{acyclic} in the sense that no variable depends on itself
w.r.t.\ this assignment (see (2) of Algorithm~\ref{algo:1}). An acyclic assignment induces
a substitution $\sigma$, as defined in (3) of Algorithm~\ref{algo:1}. We call this
substitution the \emph{current substitution}.
Initially, the current substitution maps all variables to $\top$.

The algorithm applies rules that can
\begin{enumerate}[(1)]
\item
  change an equation of the unification problem by adding non-variable
 atoms of 
  the input problem $\Gamma_0$ to one side of the equation;
\item
  introduce a new flat equation of the form $C\sqcap B\equiv B$, where $C, B$ are
  atoms of the input problem $\Gamma_0$ or $\top$;
\item
  add non-variable atoms of the input problem $\Gamma_0$ to the sets $S_X$.
\end{enumerate}
Another property that is maintained throughout the run of our algorithm is that all
equations $e$ are \emph{expanded} w.r.t.\ the current assignment in the following sense:
for all variables $X$ we have
$$
A \in S_X \wedge X\in\LV(e)\Rightarrow A\in \LA(e)\ \ \mbox{and}\ \ 
A \in S_X \wedge X\in\RV(e)\Rightarrow A\in \RA(e).
$$
Given a flat equation $e$ that contains the variable $X$, 
the \emph{expansion} of $e$ w.r.t.\ the assignment $S_X$ for $X$ is defined as follows:
if $X\in \LV(e)$ then all elements of $S_X$ are added to $\LA(e)$, and
if $X\in \RV(e)$ then all elements of $S_X$ are added to $\RA(e)$.

The following lemma is an immediate consequence of the definition of expanded equations 
and of the construction of the current substitution.

\begin{lem}\label{expanded:lem}
If the equation $C\equiv^? D$ is expanded w.r.t.\ the current assignment, then
$\LA(C\equiv^? D) = \RA(C\equiv^? D)$ implies that the current substitution $\sigma$
solves this equation, i.e., $\sigma(C) \equiv \sigma(D)$.
\end{lem}
We say that an equation $e$ is \emph{solved} if $\LA(e) = \RA(e)$. An atom 
$A\in \LA(e) \cap \RA(e)$ is called \emph{solved in $e$}; atoms $A\in \LA(e) \cup \RA(e)$ 
that are not solved in $e$ are called \emph{unsolved in $e$}. Obviously, an equation $e$ is solved iff
all atoms $A\in\LA(e) \cup \RA(e)$ are solved in $e$.

Basically, in each step, the goal-oriented algorithm considers an unsolved equation and
an unsolved atom in this equation, and tries
to solve it. Picking the unsolved equation and the unsolved atom in it
is don't care non-deterministic, i.e., there is no
need to backtrack over such a choice. Once an unsolved equation and an unsolved
atom in it was picked, don't know non-determinism comes in since there
may be several possibilities for how to solve this atom in the equation, some of which may lead to overall success 
whereas others won't. 
In some cases, however, a given equation uniquely determines the assignment for a certain
variable $X$. In this case, we make this assignment and then label the variable $X$ as \emph{finished}.
This has the effect that the set $S_X$ can no longer be extended. Initially, none of the variables
occurring in the input unification problem is labeled as finished. We say that the variable
$X$ is \emph{unfinished} if it is not labeled as finished.

%{\bf Creating an equation}
%
%The initialization of goal equations consists of initializing the proper lists at first,
%but we describe this process in a more general setting, because the procedure is used later, when
%we are to create new goal equations in the effect of applying some rules.
%
%Given an equation to be initialized: $C_1 \sqcap \cdots \sqcap C_m \approx D_1 \sqcap \cdots \sqcap D_n$,
%for each $C_i$, $1 \le i \le m$, do:
%\begin{enumerate}
%\item if $C_i$ is a non-variable atom, put it in the left list of non-variable atoms;
%%and label it {\em unsolved};
%\item if $C_i$ is a variable, put it in the left list of variables and
%put all atoms from $S_{C_i}$ into the left list of non-variable atoms;
%%while labeling them {\em unsolved};
%\end{enumerate}
%We do the same for the right side of the equation:
%for each $D_j$, $1 \le j \le n$, do:
%\begin{enumerate}
%\item if $D_j$ is a non-variable atom, put it in the left list of non-variable atoms;
%%and label them {\em unsolved};
%\item if $D_j$ is a variable, put it in the left list of variables and
%put all atoms from $S_{D_j}$ into the left list of non-variable atoms.
%%while labeling them {\em unsolved}.
%\end{enumerate}

\begin{algo}\label{goal:oriented:alg}
Let $\Gamma_0$ be a flat $\mathcal{EL}$-unification problem. We define $\Gamma := \Gamma_0$ and
$S_X := \emptyset$ for all variables $X$ occurring in $\Gamma_0$. None of these variables is labeled
as finished.

As long as $\Gamma$ contains an unsolved equation, do the following: %let $e$ be such an equation.
%, and let $A$ be an unsolved non-variable atom in $e$.
\begin{enumerate}[(1)]
\item
  If the Eager-Assignment rule applies to some equation $e$, 
  then apply it to this equation (see Figure~\ref{eager}).
\item
  Otherwise, let $e$ be an unsolved equation and $A$ an unsolved atom in $e$. If neither 
  of the rules Decomposition (see Figure~\ref{decomp})
  and Extension (see Figure~\ref{extent}) applies to $A$ in $e$, then return ``fail.'' 
  If one of these rules applies to $A$ in $e$, then (don't know) non-deterministically choose
  one of these rules and apply it.
\end{enumerate}
Once all equations of $\Gamma$ are solved, return the substitution $\sigma$ that is
induced by the current assignment.
\end{algo}

\begin{figure}[t]
\begin{boxedminipage}[t!]{\textwidth}
\begin{center}
\begin{minipage}[t!]{.95\textwidth}
The L-variant of the Eager-Assignment rule applies to the equation $e$ if there is an
unfinished variable $X\in \LV(e)$ such that
\begin{enumerate}[$\bullet$]
\item
  all variables $Z\in (\LV(e)\setminus\{X\}) \cup\RV(e)$ are finished;
\item
  $\LA(e) = \emptyset$.
\end{enumerate}
Its application sets $S_X := \RA(e)$. 
\begin{enumerate}
\item
If this makes the current assignment cyclic, then
return ``fail.'' 
\item
Otherwise, label $X$ as finished and expand all equations containing $X$
w.r.t.\ the new assignment for $X$.
\end{enumerate}
\end{minipage}
\end{center}
  \end{boxedminipage}
\caption{The Eager-Assignment rule in its L-variant. The R-variant is obtained by exchanging the 
r\^oles of the two sides of the equation.}
\label{eager}
\end{figure}

The \emph{Eager-Assignment rule} is described in Figure~\ref{eager}. Note that, after a non-failing application
of this rule, the equation it was applied to is solved since the expansion of this equation w.r.t.\ the new
assignment for $X$ adds all elements of $\RA(e)$ to $\LA(e)$. As an example, consider the equations
$$
Y\equiv^? \top,\ \ Z\equiv^? \exists\, r.\top,\ \ X\sqcap Y \equiv^? Z, 
$$
and assume that $S_X = S_Y = S_Z = \emptyset$ and none of the three variables $X, Y, Z$ is finished.
An application of the Eager-Assignment rule to the first equation labels $Y$ as finished, but does
not change anything else. The subsequent application of the Eager-Assignment rule to the second
equation changes the assignment for $Z$ to $S_Z = \{\exists\, r.\top\}$, labels $Z$ as finished, 
and expands the second and the third equation w.r.t.\ the new assignment for $Z$. Thus, we now have the equations
$$
Y\equiv^? \top,\ \ Z\sqcap \exists\, r.\top \equiv^? \exists\, r.\top,\ \ X\sqcap Y \equiv^? Z\sqcap\exists\, r.\top.
$$
Since $Y, Z$ are finished, the Eager-Assignment rule can now be applied to the third equation.
This changes the assignment for $X$ to $S_X = \{\exists\, r.\top\}$, labels $X$ as finished, 
and adds $\exists\, r.\top$ to the left-hand side of the third equation. Now all equations are
solved. The current assignment induces a substitution $\sigma$ with
$\sigma(X) = \exists\, r.\top = \sigma(Z)$ and $\sigma(Y) = \top$, which is a unifier of the
original set of equations.

\begin{figure}[t]
\begin{boxedminipage}[t!]{\textwidth}
\begin{center}
\begin{minipage}[t!]{.95\textwidth}
The L-variant of the Decomposition rule applies to the unsolved atom $A$ in the
equation $e$ if 
\begin{enumerate}[$\bullet$]
\item
  $A\in \LA(e)\setminus\RA(e)$;
\item
  $A$ is of the form $A = \exists\, r.C$; % for a role $r$ and a concept term $C$;
\item
  there is at least one atom of the form $\exists\, r.B\in \RA(e)$.
\end{enumerate}
Its application chooses (don't know) non-deterministically an atom of the form
$\exists\, r.B\in \RA(e)$ and
\begin{enumerate}[$\bullet$]
\item
adds $\exists\, r.C$ to $\RA(e)$;
\item
creates a new equation $C\sqcap B\equiv^? B$ and expands it w.r.t.\ the assignments of
all variables contained in this equation, unless this equation has already been generated before.
If the equation has already been generated before, it is not generated again.
\end{enumerate}
\end{minipage}
\end{center}
  \end{boxedminipage}
\caption{The Decomposition rule in its L-variant. The R-variant is obtained by exchanging the 
r\^oles of the two sides of the equation.}
\label{decomp}
\end{figure}

The \emph{Decomposition rule} is described in Figure~\ref{decomp}. This rule
solves the unsolved atom $A = \exists\, r.C$ by adding it to the other side. For this to be admissible,
one needs a more specific atom $\exists\, r.B$ on that side, where the ``more specific'' is meant to
hold after application of the unifier. Thus, to ensure that the unifier $\sigma$ computed by the algorithm 
satisfies $\sigma(\exists\, r.B)\sqsubseteq\sigma(\exists\, r.C)$, the rule adds the new
equation $C\sqcap B\equiv^? B$. Obviously, if the substitution $\sigma$ solves this equation, then
it satisfies $\sigma(B)\sqsubseteq\sigma(C)$, and thus
$\sigma(\exists\, r.B)\sqsubseteq\sigma(\exists\, r.C)$.
As an example, consider the equation
$$
\exists\, r.X \sqcap  \exists\, r.A\equiv^? \exists\, r.A,
$$
and assume that $S_X = \emptyset$ and that $X$ is unfinished.
An application of the L-variant of the Decomposition rule to this equation
adds $\exists\, r.X $ to the right-hand side of this equation, and thus solves it.
In addition, it generates the new equation $X\sqcap A\equiv^? A$, which is solved. 
The current assignment induces a substitution $\sigma$ with $\sigma(X) = \top$,
which solves the original equation.

The \emph{Extension rule} is described in Figure~\ref{extent}. Basically, this rule
solves the unsolved atom $A$ by extending with this atom 
the assignment of an unfinished variable contained in the
other side of the equation. As an example, consider the equation
$$
A \sqcap \exists\, r.\top \equiv^? \exists\, r.\top \sqcap X,
$$
where $A$ is a concept constant, $S_X = \emptyset$, and $X$ is unfinished.
An application of the Extension rule to $A$ in this equation extends the
assignment for $X$ to $S_X = \{A\}$, and expands this equation by adding
$A$ to the right-hand side. The equation obtained this way is solved. The substitution
$\sigma$ induced by the current assignment replaces $X$ by $A$, and solves the original equation.

\begin{figure}[t]
\begin{boxedminipage}[t!]{\textwidth}
\begin{center}
\begin{minipage}[t!]{.95\textwidth}
The L-variant of the Extension rule applies to the unsolved atom $A$
of the equation $e$ if 
\begin{enumerate}[$\bullet$]
\item
  $A\in \LA(e)\setminus\RA(e)$;
\item
  there is at least one unfinished variable $X \in \RA(e)$
\end{enumerate}
Its application chooses (don't know) non-deterministically an unfinished variable $X \in \RA(e)$
and adds $A$ to $S_X$.
\begin{enumerate}[$\bullet$]
\item
If this makes the current assignment cyclic, then return ``fail.''
\item
Otherwise, expand all equations containing $X$
w.r.t.\ the new\\ assignment for $X$.
\end{enumerate}
\end{minipage}
\end{center}
  \end{boxedminipage}
\caption{The Extension rule in its L-variant. The R-variant is obtained by exchanging the 
r\^oles of the two sides of the equation.}
\label{extent}
\end{figure}

\begin{thm}\label{goal:oriented:thm}
Algorithm~\ref{goal:oriented:alg} is an NP-algorithm for testing solvability
of flat $\mathcal{EL}$-uni\-fi\-ca\-tion problems.
\end{thm}

First, we show that the algorithm is indeed an NP-algorithm. For this, we consider all
\emph{runs} of the algorithm, where for every (don't care) non-deterministic choice
exactly one alternative is taken. Since a single rule application can
obviously be realized in polynomial time, it is sufficient to show the following lemma.

\begin{lem}[Termination]
Every run of the algorithm terminates after a polynomial number
of rule applications.
\end{lem}

\begin{proof}
Each application of the Eager-Assignment rule 
finishes an unfinished variable. Thus, since finished variables never become unfinished
again, it can only be applied $k$ times, where $k$ is the number of variables
occurring in the input unification problem $\Gamma_0$. This number is clearly
linearly bounded by the size of $\Gamma_0$.

Every application of the Decomposition rule or the Extension rule turns an unsolved
atom in an equation into a solved one, and a solved atom in an equation never becomes
unsolved again in this equation. For a fixed equation, in the worst case every atom of $\Gamma_0$ may become
an unsolved atom of the equation that needs to be solved. There is, however,
only a linear number of atoms of $\Gamma_0$. Each equation considered during the
run of the algorithm is either descended from an original equation of $\Gamma_0$, or from an equation
of the form $C\sqcap B\equiv^? B$ for atoms $\exists\, r.B$ and $\exists\, r.C$ of $\Gamma_0$.
Thus, the number of equations is also polynomially bounded by the size of $\Gamma_0$.
Overall, this shows that the Decomposition rule and the Extension rule can only be
applied a polynomial number of times.
\end{proof}

Next, we show soundness of Algorithm~\ref{goal:oriented:alg}. 
We call a run of this algorithm \emph{non-failing} if it terminates with a unification
problem containing only solved equations.

\begin{lem}[Soundness]
Let $\Gamma_0$ be a flat $\mathcal{EL}$-unification problem.
The substitution $\sigma$ returned after a successful run of Algorithm~\ref{goal:oriented:alg} 
on input $\Gamma_0$ is an $\mathcal{EL}$-unifier of $\Gamma_0$.
\end{lem}

\begin{proof}
First, note that the rules employed by Algorithm~\ref{goal:oriented:alg} indeed
preserve the two invariants mentioned before:
\begin{enumerate}[(1)]
\item
  the current assignment is always acyclic; 
\item
  all equations are expanded.
\end{enumerate}
In fact, whenever the current assignment is extended, the rules test acyclicity (and return
``fail,'' if it is not satisfied).  In addition, they expand all equations w.r.t.\ the
new assignment. 

Now, assume that the run of the algorithm has terminated with the $\mathcal{EL}$-unification
problem $\widehat{\Gamma}$, in which all equations are solved.
The first invariant ensures that the final assignment constructed by the run
is acyclic, and thus indeed induces a substitution $\sigma$. 
Because of the second invariant, Lemma~\ref{expanded:lem} applies, and thus
we know that $\sigma$ is a solution of $\widehat{\Gamma}$. 

It remains to show that the substitution $\sigma$ is also a solution of the input 
problem $\Gamma_0$. To this purpose, we take all the equations that were
considered during the run of the algorithm, i.e., present in $\Gamma_0$ or in any of the
other unification problems generated during the run.
Let $\mathcal{E}$ denote the set of these equations. We define the relation $\rightarrow$ on
$\mathcal{E}$ as follows: $e\rightarrow e'$ if $e$ was transformed into $e'$ using
one of the rules of Algorithm~\ref{goal:oriented:alg}. To be more precise, the 
Eager-Assignment rule transforms equations containing $X$ from the current unification
problem $\Gamma$ by expanding them w.r.t.\ the new assignment for $X$. 
The same is true for the Extension rule. 
The decomposition rule transforms an equation $e$ containing the unsolved atom 
$A = \exists\, r.C$ by adding this atom to the other side, which needs to contain an atom 
of the form $\exists\, r.B$. For this new equation $e'$, we have $e\rightarrow e'$. 
The decomposition rule may also generate a new equation $e''$ of the form $C\sqcap B\equiv^? B$
(if this equation was not generated before). However, we do not view this equation as a
successor of $e$ w.r.t.\ $\rightarrow$, i.e., we do \emph{not} have $e\rightarrow e''$. 
Equations $C\sqcap B\equiv^? B$ that are generated by an application of the decomposition 
rule are called \emph{D-equations}.
Equations that are elements of the input problem $\Gamma_0$ are called \emph{I-equations}.
Any equation $e'$ that is not an I-equation or a D-equation has a unique predecessor
w.r.t.\ $\rightarrow$, i.e., there is an equation $e \in \mathcal{E}$ such that
$e\rightarrow e'$.

Starting with the set $\mathcal{F} := \widehat{\Gamma}$ we will now step by step extend
$\mathcal{F}$ by a predecessor of an equation in $\mathcal{F}$ until no new predecessors
can be added. Since $\mathcal{E}$ is finite, this process terminates after a finite
number of steps. After termination we have $\mathcal{E} = \mathcal{F}$, and thus
in particular $\Gamma_0\subseteq \mathcal{F}$. 
This is due to the
fact that, for every element $e_0$ of $\mathcal{E}$, there are $n\geq 0$ elements
$e_1,\ldots,e_n\in \mathcal{E}$ such that $e_0\rightarrow e_1\rightarrow \ldots\rightarrow e_n$
and $e_n\in \widehat{\Gamma}$. Thus, it is enough to show that the set $\mathcal{F}$ satisfies
the following \emph{invariant}:
$$
(*)\ \ \ \mbox{\emph{the substitution $\sigma$ solves every equation in $\mathcal{F}$.}}
$$

Since $\sigma$ is a solution of $\widehat{\Gamma}$, this invariant is initially satisfied.
To prove that it is preserved under adding predecessors of equations in $\mathcal{F}$,
we start with the equations of minimal role depth. To be more precise, if the equation $e$ is of the
form $C\equiv^? D$, we define the \emph{role depth of $e$ w.r.t.\ $\sigma$} to be the
role depth\footnote{
see the proof of Proposition~\ref{well:founded1} for a definition.
}
of the concept term $\sigma(C)\sqcap\sigma(D)$. The strict order $\succ$ on $\mathcal{E}$ is
defined as follows: $e\succ_\sigma e'$ iff the role depth of $e$ w.r.t.\ $\sigma$ is 
larger than the role depth of $e'$ w.r.t.\ $\sigma$. We write $e\approx_\sigma e'$ if
$e$ and $e'$ have the same role depth w.r.t.\ $\sigma$.
The following is an easy consequence of the definition of $\sigma$ and of our rules:
$$
(**)\ \ \ e_1\rightarrow e_2 \rightarrow \ldots \rightarrow e_n\ \ \mbox{implies}\ \ 
         e_1\approx_\sigma e_2 \approx_\sigma\ldots \approx_\sigma e_n.
$$
Assume that we have already constructed a set $\mathcal{F}$ such that the invariant $(*)$
is satisfied. Let $e'$ be an equation in $\mathcal{F}$ such that
\begin{enumerate}[$\bullet$]
\item
  there is an $e\in \mathcal{E}\setminus\mathcal{F}$ with $e\rightarrow e'$;
\item
  $e'$ is of minimal role depth with this property, 
  i.e., if $f'\in \mathcal{F}$ is such that $e'\succ f'$ and $f'$ has a predecessor
  $f$ w.r.t.\  $\rightarrow$, then $f\in \mathcal{F}$.
\end{enumerate}
If no such equation $e'$ exists, then we are finished, and we have $\mathcal{E} = \mathcal{F}$.
Otherwise, let $e'$ be such an equation and $e$ its predecessor w.r.t.\  $\rightarrow$.
We add $e$ to $\mathcal{F}$. In order to show that the invariant $(*)$ is still satisfied,
we make a case distinction according to which rule was applied to $e$ to produce $e'$:
\begin{enumerate}[(1)]
 \item \emph{Eager-Assignment.} 
By an application of this rule, 
the assignment for $X$ is modified from $S_X = \emptyset$ to $S_X = \{A_1,\ldots,A_n\}$,
where $A_1, \ldots, A_n$ are non-variable atoms. In addition,
$X$ is labeled as finished. Since the assignment of a finished variable cannot be changed
anymore, we know that we also have $S_X = \{A_1,\ldots,A_n\}$ in the final assignment, and
thus $\sigma(X) = \sigma(A_1)\sqcap\ldots\sqcap\sigma(A_n)$. The rule modifies equations
as follows: all equations containing $X$ are expanded w.r.t.\ the assignment $S_X = \{A_1,\ldots,A_n\}$.
Since $e$ is transformed into $e'$ using this rule, it must contain $X$.
We assume for the sake of simplicity that $X$ is contained in the
left-hand side of $e$, but not in the right-hand side, i.e., $e$ is of the form
$C\sqcap X \equiv^? D$  and the new equation $e'\in \Gamma'$ obtained from $e$ is
$C\sqcap X \sqcap A_1\sqcap\ldots\sqcap A_n  \equiv^? D$. Since $\sigma$ solves $e'$, we have
$\sigma(D)\equiv \sigma(C\sqcap X \sqcap A_1\sqcap\ldots\sqcap A_n) \equiv
\sigma(C)\sqcap \sigma(A_1)\sqcap\ldots\sqcap\sigma(A_n)\sqcap \sigma(A_1)\sqcap\ldots\sqcap\sigma(A_n)\equiv
\sigma(C)\sqcap \sigma(A_1)\sqcap\ldots\sqcap\sigma(A_n)\equiv \sigma(C\sqcap X)$, which shows
that $\sigma$ also solves $e$.

\item \emph{Decomposition.} 
Without loss of generality, we consider the L-variant of this rule. Thus, the equation $e$
is of the form $D\sqcap\exists\, r.C\equiv^? E\sqcap\exists\, r.B$, and $e'$ is obtained from
$e$ by adding $\exists\, r.C$ to the right-hand side, i.e., $e'$ is of the form
$D\sqcap\exists\, r.C\equiv^? E\sqcap\exists\, r.B\sqcap\exists\, r.C$. We know that $\sigma$
solves $e'$. Thus, if we can show $\sigma(B)\sqsubseteq\sigma(C)$, then we have
$\sigma(D\sqcap\exists\, r.C) \equiv \sigma(E)\sqcap\sigma(\exists\, r.B)\sqcap\sigma(\exists\, r.C)\equiv
\sigma(E)\sqcap\sigma(\exists\, r.B)$, which shows that $\sigma$ solves $e$.

Consequently, it is sufficient to prove $\sigma(B)\sqsubseteq\sigma(C)$.
The Decomposition rule also generates
the equation $C\sqcap B\equiv^? B$ and expands it w.r.t.\ the assignments of all the variables
contained in this equation, unless this equation has already been generated before. Thus, either
this application or a previous one of the Decomposition rule has generated the equation
$C\sqcap B\equiv^? B$, and then expanded it (w.r.t.\ the current assignment at that time) to an 
equation $e_1$.  
Since atoms are never removed from an assignment, the atoms present in the
assignment at the time when the Decomposition rule generated the equation $C\sqcap B\equiv^? B$ are
also present in the final assignment used to define the substitution $\sigma$. Thus, if we
can show that $\sigma$ solves $e_1$, then we have also shown that $\sigma$ solves $C\sqcap B\equiv^? B$,
and thus satisfies $\sigma(B)\sqsubseteq\sigma(C)$.  
 
Since equations are never completely removed by our rules, but only modified,
there is a sequence of equations 
$e_1\rightarrow e_2 \rightarrow \ldots \rightarrow e_n$ 
such that $e_n\in \widehat{\Gamma}$. Property $(**)$ thus yields 
$e_1\approx_\sigma e_2 \approx_\sigma\ldots \approx_\sigma e_n$.
In addition, the role depth of
$C\sqcap B\equiv^? B$ w.r.t.\ $\sigma$ is the same as the role depth of $e_1$ w.r.t.\ $\sigma$.
Consequently, we have $e' \succ e_i$ for all $i, 1\leq i\leq n$. Now, assume that $e_1\not\in\mathcal{F}$.
Then there is an $i > 1$ such that $e_i\in \mathcal{F}$, but $e_{i-1}\in \mathcal{E}\setminus\mathcal{F}$.
This contradicts our assumption that $e'$ is minimal. Thus, we have shown that $e_1\in \mathcal{F}$,
and this implies that $\sigma$ solves $e_1$.

Overall, this finishes the proof that $\sigma$ solves $e$.

\item \emph{Extension.} 
By an application of this rule, the assignment for $X$ is modified by adding a non-variable 
atom $A$ to it. Since atoms are never removed from an assignment,
we know that we also have $A\in S_X$ in the final assignment, and
thus $\sigma(X) \sqsubseteq \sigma(A)$. The rule modifies equations
as follows: all equations containing $X$ are expanded w.r.t.\ the new assignment 
for $X$. Since $e$ is transformed into $e'$ using this rule, it must contain $X$.
We assume for the sake of simplicity that $X$ is contained in the
left-hand side of $e$, but not in the right-hand side, i.e., $e$ is of the form
$C\sqcap X \equiv^? D$  and the new equation $e'$ obtained from $e$ is
$C\sqcap X \sqcap A \equiv^? D$. Since $\sigma$ solves $e'$, we have
$\sigma(D)\equiv \sigma(C\sqcap X \sqcap A) \equiv
\sigma(C)\sqcap \sigma(X)\sqcap \sigma(A)\equiv
\sigma(C)\sqcap \sigma(X) \equiv \sigma(C\sqcap X)$, which shows
that $\sigma$ also solves $e$.
\end{enumerate}
To sum up, we have shown that the invariant $(*)$ is still satisfied after
adding $e$ to $\mathcal{F}$. This completes the proof of soundness of our procedure.
\end{proof}

It remains to show completeness of Algorithm~\ref{goal:oriented:alg}. Thus, assume that
the input unification problem $\Gamma_0$ is solvable. Proposition~\ref{minimal:unif:prop}
tells us that $\Gamma_0$ then has an is-minimal reduced ground unifier $\gamma$, 
and Proposition~\ref{main:prop} implies that, for every variable $X$ occurring in $\Gamma_0$,
there is a set $S_X^\gamma$ of non-variable atoms of $\Gamma_0$ such that
$$
\gamma(X) \equiv \gamma({\bigsqcap}S_X^\gamma),
$$
where, for a set of non-variable atoms $S$ of $\Gamma_0$, the expression ${\bigsqcap}S$
denotes the conjunction of the elements of $S$ (where the empty conjunction is $\top$).

\begin{lem}[Completeness]
Let $\Gamma_0$ be a flat $\mathcal{EL}$-unification problem, and assume that
$\gamma$ is an is-minimal reduced ground unifier of $\Gamma_0$.
Then there is a successful run of Algorithm~\ref{goal:oriented:alg} on input $\Gamma_0$
that returns a unifier $\sigma$ that is equivalent to $\gamma$, i.e., satisfies
$\sigma(X)\equiv\gamma(X)$ for all variables $X$ occurring in $\Gamma_0$.
\end{lem}

\proof
%\begin{proof}
The algorithm  starts with $\Gamma := \Gamma_0$ and the initial assignment
$S_X :=\emptyset$ for all variables $X$ occurring in $\Gamma_0$.
It then applies rules that change $\Gamma$ and the current assignment as
long as the problem $\Gamma$ contains an unsolved equation.

We use $\gamma$ to guide the (don't know) non-deterministic choices to be made
during the algorithm. We show that this ensures that the run of the algorithm 
generated this way \emph{does not fail} and that the following \emph{invariants} are satisfied
throughout this run:
\begin{enumerate}[(${\protect I}_1$)]
\item
   $\gamma$ is a unifier of $\Gamma$;
\item
   for all atoms $B \in S_X$ there exists an atom $A \in S_X^\gamma$ such that 
   $\gamma(A)\sqsubseteq\gamma(B)$;
\item
   for all finished variables $X$ we have $\gamma(X)\equiv\gamma({\bigsqcap}S_X)$.
\end{enumerate}
Before constructing a run that satisfies these invariants, let us point out two interesting
consequences that they have:
\begin{enumerate}[($C_1$)]
\item
  \emph{The current assignment is always acyclic.} 
  In fact, if $X$ directly depends on $Y$, then there is an atom $B\in S_X$ that has the form
  $B = \exists\, r.Y$ for some role name $r$. Invariant $I_2$ then implies that there is
  an $A \in S_X^\gamma$ such that $\gamma(X) \sqsubseteq\gamma(A)\sqsubseteq\gamma(B) = \exists\, r.\gamma(Y)$.
  Thus, if $X$ depends on $X$, then
  there are $k\geq 1$ role  names $r_1,\ldots,r_k$ such that
  $\gamma(X)\sqsubseteq \exists\, r_1.\cdots\exists\, r_k.\gamma(X)$, which is impossible.
\item%[($C_2$)]
  \emph{For each variable $X$ occurring in $\Gamma_0$, we have $\gamma(X)\sqsubseteq\sigma(X)$,
  where $\sigma$ is the current substitution induced by the current assignment.}
  This is again a consequence of invariant $I_2$. Indeed, recall that the fact that the current
  assignment is acyclic implies that there is a strict linear order $>$ on the variables occurring 
  in $\Gamma$ such that $X > Y$ if $X$ depends on $Y$. The current substitution $\sigma$ is defined along
  this order. We prove $\gamma(X)\sqsubseteq\sigma(X)$ by induction on this order. 

  Consider the
  least variable $X$. If $S_X = \emptyset$, then $\sigma(X) = \top$, and thus 
  $\gamma(X)\sqsubseteq\sigma(X)$ is trivially satisfied. Otherwise, we know, for every $B\in S_X$,
  that it does not contain any variables, which implies that $\sigma(B) = B = \gamma(B) \sqsupseteq
  \gamma(A)$ for some atom $A\in S_X^\gamma$. Obviously, this yields
  $\sigma(X) = \sigma({\bigsqcap}S_X) \sqsupseteq \gamma({\bigsqcap}S_X^\gamma) = \gamma(X)$.

  Now, assume that $\gamma(Y)\sqsubseteq\sigma(Y)$ holds for all variables $Y < X$. Since the concept
  constructors of $\mathcal{EL}$ are monotone  w.r.t.\ subsumption, this implies
  $\gamma(C)\sqsubseteq\sigma(C)$ for all concept terms $C$ containing only variables smaller than $X$.
  If $S_X$ is empty, then $\sigma(X)= \top\sqsupseteq \gamma(X)$ is trivially satisfied.
  Otherwise, we know, for every $B\in S_X$, that it contains only variables smaller than $X$.
  This yields $\sigma(B)\sqsupseteq \gamma(B) \sqsupseteq \gamma(A)$ for some atom $A\in S_X^\gamma$.
  Again, this implies 
  $\sigma(X) = \sigma({\bigsqcap}S_X) \sqsupseteq \gamma({\bigsqcap}S_X^\gamma) = \gamma(X)$.
\end{enumerate}
Since $\gamma$ was assumed to be an is-minimal unifier of $\Gamma_0$, the consequence $C_2$
implies that $\sigma$ can only be a unifier of
$\Gamma_0$ if $\sigma$ is equivalent to $\gamma$. If the run has terminated successfully,
then the final substitution $\sigma$ obtained by the run is a unifier of $\Gamma_0$ (due
to soundness). Thus, in this case the computed unifier $\sigma$ is indeed equivalent
to $\gamma$.
Consequently, to prove the lemma, it is sufficient to \emph{construct a non-failing run of the algorithm that 
satisfies the above invariants.}

The invariants are initially satisfied since $\gamma$ is a unifier of $\Gamma_0$,
the initial assignment for all variables $X$ occurring in $\Gamma_0$ is $S_X = \emptyset$,
and there are no finished variables.
Now, assume that, by application of the rules of Algorithm~\ref{goal:oriented:alg}, we have 
constructed a unification problem $\Gamma$ and a current assignment such that the invariants
are satisfied. 
\begin{enumerate}[(1)]
\item
If all equations in $\Gamma$ are solved, then the run terminates
successfully, and we are done.
\item
If there is an unsolved equation to which the \emph{Eager-Assignment rule} applies, then the
algorithm picks such an equation $e$ and applies this rule to it. Without loss of
generality, we assume that the L-variant of the rule is applied.
The selected equation $e$ is of the form 
$$
X \sqcap Z_1 \sqcap \ldots \sqcap Z_k \equiv^? A_1 \sqcap \ldots \sqcap A_n \sqcap Y_1 \sqcap \ldots \sqcap Y_m,
$$
where $A_1, \dots, A_n$ are non-variable atoms, 
and $Y_1, \dots, Y_m, Z_1, \dots Z_k$ are finished variables.
Because the left-hand side of the equation does not contain any non-variable atoms, 
we know that $S_X = S_{Z_1} = \ldots = S_{Z_k} = \emptyset$ (since the algorithm keeps all equations expanded).
Since $Z_1, \dots, Z_k$ are finished, we thus have 
$\gamma(Z_1) = \ldots = \gamma(Z_k) = \top$ (by invariant $I_3$).
We also know that $S_{Y_i}\subseteq \{A_1,\ldots,A_n\}$ for all $i, 1\leq i\leq m$. Since
the variables $Y_i$ are finished, invariant $I_3$ implies that
$\gamma(Y_i)\sqsupseteq \gamma(A_1)\sqcap\ldots\sqcap\gamma(A_n)$.

The new assignment for $X$ is $S_X = \{ A_1, \ldots, A_n\}$, 
all equations containing $X$ are expanded w.r.t.\ this assignment,
and $X$ becomes a finished variable.
First, we show that $I_3$ is satisfied. Nothing has changed for the variables that
were already finished before the application of the rule. However, $X$ is now also finished.
Thus, we must show that $\gamma(X)\equiv\gamma({\bigsqcap}S_X)$. We know that $\gamma$
solves the equation $e$ (by $I_1$). This yields
$\gamma(X) \equiv \gamma(X) \sqcap \gamma(Z_1)\sqcap\ldots\sqcap\gamma(Z_k) 
\equiv \gamma(A_1)\sqcap\ldots\sqcap\gamma(A_n)\sqcap \gamma(Y_1)\sqcap\ldots\sqcap\gamma(Y_m)
\equiv \gamma(A_1)\sqcap\ldots\sqcap\gamma(A_n) = \gamma({\bigsqcap}S_X)$.
Regarding $I_2$, the only assignment that was changed is the one for $X$. 
Since the new assignment for $X$ is $S_X = \{ A_1, \ldots, A_n\}$, and we have already shown
that $\gamma(X) \equiv \gamma(A_1)\sqcap\ldots\sqcap\gamma(A_n)$, the invariant $I_2$ holds
by Corollary~\ref{cor2}. Note that this also implies that the new assignment is acyclic,
and thus the application of the Eager-Assignment rule does not fail.
Finally, consider the invariant $I_1$. The rule application modifies equations containing
$X$ by adding the atoms $A_1, \ldots, A_n$. Since $\gamma(X) \equiv \gamma(A_1)\sqcap\ldots\sqcap\gamma(A_n)$,
an equation that was solved by $\gamma$ before this modification, is also
solved by $\gamma$ after this modification. To sum up, we have shown that the application
of the Eager-Assignment rule does not fail and preserves the invariants.

\item
If there is \emph{no} unsolved equation to which the \emph{Eager-Assignment rule} applies, then the
algorithm picks an unsolved equation $e$ and an unsolved atom $A$ occurring in this equation.
We must show that we can apply either the Decomposition or the Extension rule to $A$ in $e$
such that the invariants stay satisfied. Without loss of generality, we assume that the unsolved
atom $A$ occurs on the left-hand side of the equation $e$.

\begin{enumerate}[(a)]
\item
First, assume that \emph{$A$ is an existential restriction $A = \exists\, r.C$}. 
The selected unsolved equation $e$ is thus of the form
$$
\exists\, r.C \sqcap A_1 \sqcap \ldots \sqcap A_m \equiv B_1 \sqcap \ldots \sqcap B_n,
$$
where $A_1, \dots, A_m$ and $B_1, \ldots, B_n$ are (variable or non-variable) atoms and
$\exists\, r.C\not\in \{B_1, \ldots, B_n\}$. Since $\gamma$ solves this equation (by invariant $I_1$),
Corollary~\ref{cor2} implies that there must be an $i, 1\leq i\leq n$, such that 
$\gamma(B_i) \sqsubseteq \exists\, r.\gamma(C)$.
\begin{enumerate}[(i)]
\item
  If $B_i$ is an existential restriction $B_i = \exists\, r.B$, then we have $\gamma(B)\sqsubseteq \gamma(C)$. 
  We apply the Decomposition rule to $A$ and $B_i$. The application of this rule modifies
  the equation $e$ to an equation $e'$ by adding the atom $A$ to the right-hand side.
  In addition, it generates the equation $C\sqcap B\equiv^? B$ and expands it w.r.t.\ the assignments
  of all variables contained in this equation (unless this equation has been generated before). 
  After the application of this rule, the invariants
  $I_2$ and $I_3$ are still satisfied since the current assignments and the set of finished
  variables remain unchanged. Regarding invariant $I_1$, since $\gamma$ solves $e$, it obviously
  also solves $e'$ due to the fact that $\gamma(B_i) \sqsubseteq \gamma(A)$ and $B_i$ is a conjunct
  on the right-hand side of $e$. In addition, $\gamma(B)\sqsubseteq \gamma(C)$ implies that
  $\gamma$ also solves the equation $C\sqcap B\equiv^? B$. Since invariant $I_2$ is satisfied, this 
  implies that $\gamma$ also solves the equation obtained from $C\sqcap B\equiv^? B$ by 
  expanding it w.r.t.\ the assignments of all variables contained in it.
\item
  Assume that there is no $i, 1\leq i\leq n$, such that $B_i$ is an existential restriction satisfying
  $\gamma(B_i) \sqsubseteq \gamma(A)$. Thus, if $B_i$ is such that $\gamma(B_i) \sqsubseteq \gamma(A)$,
  then we know that $B_i = X$ is a variable. We want to apply the Extension rule to $A$ and $X$.
  To be able to do this, we must first show that $X$ is not a finished variable. 

  Thus, assume that $X$ is finished, and let  
  $S_X  = \{C_1, \dots, C_\ell\}$. Invariant $I_3$ yields
  $\gamma(C_1)\sqcap\ldots\sqcap\gamma(C_\ell) = \gamma(X) = \gamma(B_i) 
   \sqsubseteq \gamma(A) = \exists\, r.\gamma(C)$, and thus there is a $j, 1\leq j\leq \ell$,
  such that $\gamma(C_j) \sqsubseteq \gamma(A)$. Since $A$ is an existential restriction, the
  non-variable atom $C_j$ must also be an existential restriction, and since the equation
  $e$ is expanded, $C_j\in S_X$ occurs on the right-hand side of this equation.
  This contradicts our assumption that there is no such existential restriction on the
  right-hand side. Thus, we have shown that $X$ is not finished, which means that we can
  apply the Extension rule to $A$ and $X$.

  The application of this rule adds the atom $A$ to the assignment for $X$, and it expands
  all equations containing $X$ w.r.t.\ this new assignment, i.e., it adds $A$ to the left-hand side
  and/or right-hand
  side of an equation whenever $X$ is contained in this side. Since we know that $\gamma(X)\sqsubseteq\gamma(A)$,
  it is easy to see that, if $\gamma$ solves an equation before this expansion, it also solves
  it after the expansion. Thus invariant $I_1$ is satisfied. Invariant $I_2$ also remains satisfied.
  In fact, if $S_{X}^\gamma = \{D_1, \dots, D_k\}$, then
  $\gamma(D_1)\sqcap\ldots\sqcap\gamma(D_k) = \gamma(X) \sqsubseteq \gamma(A)$ implies that there
  is a $j, 1\leq j\leq \ell$, such that $\gamma(D_j) \sqsubseteq \gamma(A)$. The fact that $I_2$
  is satisfied by the new assignment also implies that this new assignment is acyclic,
  and thus the application of the Extension rule does not fail. Invariant $I_3$ is still
  satisfied since $X$ is not finished, and the assignments of variables different from $X$
  were not changed.
\end{enumerate}

\item
Second, assume that \emph{$A$ is a concept name}.
The selected unsolved equation $e$ is thus of the form
$$
A \sqcap A_1 \sqcap \ldots \sqcap A_m \equiv B_1 \sqcap \ldots \sqcap B_n,
$$
where $A_1, \dots, A_m$ and $B_1, \ldots, B_n$ are (variable or non-variable) atoms, and
$A\not\in \{B_1, \ldots, B_n\}$. Since $\gamma$ solves this equation (by invariant $I_1$),
Corollary~\ref{cor2} implies that there must be an $i, 1\leq i\leq n$, such that
$\gamma(B_i) \sqsubseteq \gamma(A) = A$.
Since $A\not\in \{B_1, \ldots, B_n\}$, we know that $B_i = X$ is a variable.
We want to apply the Extension rule to $A$ and $X$.
To be able to do this, we must first show that $X$ is not a finished variable.

Thus, assume that $X$ is finished, and let $S_X  = \{C_1, \dots, C_\ell\}$. 
Invariant $I_3$ yields $\gamma(C_1)\sqcap\ldots\sqcap\gamma(C_\ell) = \gamma(X) = \gamma(B_i)
\sqsubseteq \gamma(A) = A$, and thus there is a $j, 1\leq j\leq \ell$,
such that $\gamma(C_j) \sqsubseteq A$. Since $A$ is a concept name, the
non-variable atom $C_j$ must actually be equal to $A$, and since the equation
$e$ is expanded, $C_j = A\in S_X$ occurs on the right-hand side of this equation.
This contradicts our assumption that $A$ is an unsolved atom.
Thus, we have shown that $X$ is not finished, which means that we can
apply the Extension rule to $A$ and $X$.
The application of this rule adds the atom $A$ to the assignment for $X$, and it expands
all equations containing $X$ w.r.t.\ this new assignment. The proof that this rule application
does not fail and preserves the invariants is identical to the one for the case where
$A$ was an existential restriction. \hfill \qed
\end{enumerate}
\end{enumerate}
%\end{proof}

To sum up, we have shown that Algorithm~\ref{goal:oriented:alg} always terminates
(in non-deterministic polynomial time) and that it is sound and complete. This
finishes the proof of Theorem~\ref{goal:oriented:thm}.

\section{Unification in semilattices with monotone operators}\label{eq:th:sect}

Unification problems and their types were originally not introduced for Description Logics,
but for equational theories \cite{BaSn01}.  
In this section, we show that the above
results for unification in $\mathcal{EL}$ can actually be viewed as results
for an equational theory. 
As shown in \cite{SoSt08}, the equivalence problem for
$\mathcal{EL}$-concept terms corresponds to the word problem for the equational theory
of semilattices with monotone operators. In order to define this theory, we consider
a signature $\Sigma_{\SLO}$ consisting of a binary function symbol $\wedge$, a constant symbol $1$,
and finitely many unary function symbols $f_1, \ldots, f_n$. Terms can then be built using
these symbols and additional variable symbols and free constant symbols.

\begin{defi}
The equational theory of \emph{semilattices with monotone operators} is defined
by the following identities:
$$
\begin{array}{l@{\ }l}
\SLO := &\{x\wedge(y\wedge z) = (x\wedge y)\wedge z,\ x\wedge y = y\wedge x,\ x\wedge x = x,\ x\wedge 1 = x\}\ \cup\\[.5em]
&\{f_i(x\wedge y)\wedge f_i(y) = f_i(x\wedge y) \mid  1\leq i\leq n\}
\end{array}
$$
\end{defi}

A given
$\mathcal{EL}$-concept term $C$ using only roles
$r_1,\ldots,r_n$ can be translated into a term $t_C$ over the signature $\Sigma_{\SLO}$
by replacing each concept constant $A$ by a corresponding free constant $a$, each concept
variable $X$ by a corresponding variable $x$, $\top$ by $1$, $\sqcap$ by $\wedge$,
and $\exists\, r_i$ by $f_i$. For example, the $\mathcal{EL}$-concept term
$C = A \sqcap \exists\, r_1.\top\sqcap\exists\, r_3.(X\sqcap B)$ is translated
into $t_C = a\wedge f_1(1)\wedge f_3(x\wedge b)$. Conversely, any term over the signature $\Sigma_{\SLO}$
can be translated back into an $\mathcal{EL}$-concept term.

\begin{lem}
Let $C, D$ be $\mathcal{EL}$-concept term using only roles $r_1,\ldots,r_n$.
Then $C\equiv D$ iff $t_C =_\SLO t_D$.
\end{lem}

As an immediate consequence of this lemma, we have that unification in the
DL $\mathcal{EL}$ corresponds to unification modulo the equational theory
$\SLO$. Thus, Theorem~\ref{type:zero:thm} implies that $\SLO$ has unification
type zero, and Theorem~\ref{NP:complete:thm} implies that $\SLO$-unification
is NP-complete.

\begin{cor}
The equational theory $\SLO$ of semilattices with monotone operators
has unification type zero, and deciding solvability of an
$\SLO$-unification problem is an NP-complete problem.
\end{cor}

Since the unification problem introduced in Theorem~\ref{type:zero:thm}
contains only one role $r$, this is already true in the presence of
a single monotone operator.

\section{Conclusion}

In this paper, we have shown that unification in the DL $\mathcal{EL}$
is of type zero and NP-complete.
There are interesting differences between the behavior of $\mathcal{EL}$
and the closely related DL $\mathcal{FL}_0$ w.r.t.\ unification and matching. 
Though the unification types coincide for these two DLs, the complexities of the decision
problems differ: $\mathcal{FL}_0$-unification is ExpTime-complete, and thus considerably
harder than $\mathcal{EL}$-unification. In contrast, $\mathcal{FL}_0$-matching is polynomial,
and thus considerably easier than $\mathcal{EL}$-matching, which is NP-complete.
In addition to showing the complexity upper bound for $\mathcal{EL}$-unification
by a simple ``guess and then test'' NP-algorithm, we have also developed a more goal-oriented
NP-algorithm that makes (don't know) non-deterministic decisions (i.e., ones that require
backtracking) only if they are triggered by unsolved atoms in the
unification problem.

As future work, we will consider also unification of concept terms
for other members of the $\mathcal{EL}$-family of DLs \cite{BaBL05}. In addition, we
will investigate unification modulo more expressive terminological formalisms.
On the practical side, we will optimize and implement the goal-oriented  $\mathcal{EL}$-unification
algorithm developed in Section~\ref{goal:oriented:sect}. We intend to test the usefulness
of this algorithm for the purpose on finding redundancies in  $\mathcal{EL}$-based ontologies
by considering extensions of the medical ontology {\sc Snomed~ct}. 
For example, in \cite{COQ+07}, two different extensions of {\sc Snomed~ct} by so-called post-coordinated
concepts were considered. The authors used an (incomplete) equivalence test to find out how large
the overlap between the two extensions is (i.e., how many of the new concepts belonged to
both extensions). As pointed out in the introduction, the equivalence test cannot deal with
situations where different knowledge engineers use different names for concepts, 
or model on different levels of granularity. We want to find out whether using unifiability
rather than equivalence finds more cases of overlapping concepts. Of course, in the case of unification one
may also obtain false positives, i.e., pairs of concepts that are unifiable, but are not
meant to represent the same (intuitive) concept. It is also important to find out how
often this happens. Another problem to be dealt with in this application is
the development of heuristics for choosing the pairs of concepts to be tested for
unifiability and for deciding which concept names are turned into variables.

%\bibliographystyle{plain}
%\bibliography{medium-string,dl,unification}

\begin{thebibliography}{10}

\bibitem{Baad89b}
Franz Baader.
\newblock Characterizations of unification type zero.
\newblock In N.~Dershowitz, editor, {\em Proceedings of the 3rd International
  Conference on Rewriting Techniques and Applications}, volume 355 of {\em
  Lecture Notes in Computer Science}, pages 2--14, Chapel Hill, North Carolina,
  1989. Springer-Verlag.

\bibitem{Baad89}
Franz Baader.
\newblock Unification in commutative theories.
\newblock {\em J. Symbolic Computation}, 8(5):479--497, 1989.

\bibitem{Baad90c}
Franz Baader.
\newblock Terminological cycles in {KL-ONE}-based knowledge representation
  languages.
\newblock In {\em Proc.\ of the 8th Nat.\ Conf.\ on Artificial Intelligence
  (AAAI'90)}, pages 621--626, Boston (Ma, USA), 1990.

\bibitem{Baad03e}
Franz Baader.
\newblock Terminological cycles in a description logic with existential
  restrictions.
\newblock In Georg Gottlob and Toby Walsh, editors, {\em Proc.\ of the 18th
  Int.\ Joint Conf.\ on Artificial Intelligence (IJCAI~2003)}, pages 325--330,
  Acapulco, Mexico, 2003. Morgan Kaufmann, Los Altos.

\bibitem{BaBL05}
Franz Baader, Sebastian Brandt, and Carsten Lutz.
\newblock Pushing the {$\mathcal{EL}$} envelope.
\newblock In Leslie~Pack Kaelbling and Alessandro Saffiotti, editors, {\em
  Proc.\ of the 19th Int.\ Joint Conf.\ on Artificial Intelligence
  (IJCAI~2005)}, pages 364--369, Edinburgh (UK), 2005. Morgan Kaufmann, Los
  Altos.

\bibitem{BCNMP03}
Franz Baader, Diego Calvanese, Deborah McGuinness, Daniele Nardi, and Peter~F.
  Patel-Schneider, editors.
\newblock {\em The Description Logic Handbook: Theory, Implementation, and
  Applications}.
\newblock Cambridge University Press, 2003.

\bibitem{BaKu00}
Franz Baader and Ralf K{\"u}sters.
\newblock Matching in description logics with existential restrictions.
\newblock In {\em Proc.\ of the 7th Int.\ Conf.\ on Principles of Knowledge
  Representation and Reasoning (KR~2000)}, pages 261--272, 2000.

\bibitem{BaKu01}
Franz Baader and Ralf K{\"u}sters.
\newblock Unification in a description logic with transitive closure of roles.
\newblock In Robert Nieuwenhuis and Andrei Voronkov, editors, {\em Proc.\ of
  the 8th Int.\ Conf.\ on Logic for Programming and Automated Reasoning
  (LPAR~2001)}, volume 2250 of {\em Lecture Notes in Artificial Intelligence},
  pages 217--232, Havana, Cuba, 2001. Springer-Verlag.

\bibitem{BKBM99}
Franz Baader, Ralf K{\"u}sters, Alex Borgida, and Deborah~L. McGuinness.
\newblock Matching in description logics.
\newblock {\em J.\ of Logic and Computation}, 9(3):411--447, 1999.

\bibitem{BaKM99}
Franz Baader, Ralf K{\"u}sters, and Ralf Molitor.
\newblock Computing least common subsumers in description logics with
  existential restrictions.
\newblock In {\em Proc.\ of the 16th Int.\ Joint Conf.\ on Artificial
  Intelligence (IJCAI'99)}, pages 96--101, 1999.

\bibitem{BaMo09}
Franz Baader and Barbara Morawska.
\newblock Unification in the description logic {$\mathcal{EL}$}.
\newblock In Ralf Treinen, editor, {\em Proc.\ of the 20th Int.\ Conf.\ on
  Rewriting Techniques and Applications (RTA 2009)}, volume 5595 of {\em
  Lecture Notes in Computer Science}, pages 350--364. Springer-Verlag, 2009.

\bibitem{BaNa00}
Franz Baader and Paliath Narendran.
\newblock Unification of concepts terms in description logics.
\newblock {\em J.\ of Symbolic Computation}, 31(3):277--305, 2001.

\bibitem{BaNiLong98}
Franz Baader and Tobias Nipkow.
\newblock {\em Term Rewriting and All That}.
\newblock Cambridge University Press, United Kingdom, 1998.

\bibitem{BaNu03}
Franz Baader and Werner Nutt.
\newblock Basic description logics.
\newblock In {\em \cite{BCNMP03}}, pages 43--95. 2003.

\bibitem{BaST07}
Franz Baader, Baris Sertkaya, and Anni-Yasmin Turhan.
\newblock Computing the least common subsumer w.r.t.\ a background terminology.
\newblock {\em J.\ of Applied Logic}, 5(3):392--420, 2007.

\bibitem{BaSn01}
Franz Baader and Wayne Snyder.
\newblock Unification theory.
\newblock In J.A. Robinson and A.~Voronkov, editors, {\em Handbook of Automated
  Reasoning}, volume~I, pages 447--533. Elsevier Science Publishers, 2001.

\bibitem{Bran04}
Sebastian Brandt.
\newblock Polynomial time reasoning in a description logic with existential
  restrictions, {GCI} axioms, and---what else?
\newblock In Ramon~L{\'o}pez de~M{\'a}ntaras and Lorenza Saitta, editors, {\em
  Proc.\ of the 16th Eur.\ Conf.\ on Artificial Intelligence (ECAI~2004)},
  pages 298--302, 2004.

\bibitem{COQ+07}
James~R. Campbell, Alejandro Lopez~Osornio, Fernan de~Quiros, Daniel Luna, and
  Guillermo Reynoso.
\newblock Semantic interoperability and {SNOMED~CT}: A case study in clinical
  problem lists.
\newblock In K.A. Kuhn, J.R. Warren, and T.-Y. Leong, editors, {\em Proc.\ of
  the 12th World Congress on Health (Medical) Informatics {(MEDINFO 2007)}},
  pages 2401--2402. IOS Press, 2007.

\bibitem{Ghil00}
Silvio Ghilardi.
\newblock Best solving modal equations.
\newblock {\em Ann. Pure Appl. Logic}, 102(3):183--198, 2000.

\bibitem{HoPH03}
Ian Horrocks, Peter~F. Patel-Schneider, and Frank van Harmelen.
\newblock From {SHIQ} and {RDF} to {OWL}: The making of a web ontology
  language.
\newblock {\em Journal of Web Semantics}, 1(1):7--26, 2003.

\bibitem{HoST00}
Ian Horrocks, Ulrike Sattler, and Stefan Tobies.
\newblock Practical reasoning for very expressive description logics.
\newblock {\em J.\ of the Interest Group in Pure and Applied Logic},
  8(3):239--264, 2000.

\bibitem{JoKi91}
Jean-Pierre Jouannaud and Claude Kirchner.
\newblock Solving equations in abstract algebras: {A} rule-based survey of
  unification.
\newblock In J.-L. Lassez and G.~Plotkin, editors, {\em Computational Logic:
  Essays in Honor of A. Robinson}. MIT Press, Cambridge, MA, 1991.

\bibitem{KaNi03}
Yevgeny Kazakov and Hans de~Nivelle.
\newblock Subsumption of concepts in {$\mathcal{FL}_0$} for (cyclic)
  terminologies with respect to descriptive semantics is {PSPACE}-complete.
\newblock In {\em Proc.\ of the 2003 Description Logic Workshop (DL~2003)}.
  CEUR Electronic Workshop Proceedings, http://CEUR-WS.org/Vol-81/, 2003.

\bibitem{Kues01}
Ralf K{\"u}sters.
\newblock {\em Non-standard Inferences in Description Logics}, volume 2100 of
  {\em Lecture Notes in Artificial Intelligence}.
\newblock Springer-Verlag, 2001.

\bibitem{LeBr85}
Hector~J. Levesque and Ron~J. Brachman.
\newblock A fundamental tradeoff in knowledge representation and reasoning.
\newblock In Ron~J. Brachman and Hector~J. Levesque, editors, {\em Readings in
  Knowledge Representation}, pages 41--70. Morgan Kaufmann, Los Altos, 1985.

\bibitem{Nebe90}
Bernhard Nebel.
\newblock Terminological reasoning is inherently intractable.
\newblock {\em Artificial Intelligence}, 43:235--249, 1990.

\bibitem{ReHo97}
Alan Rector and Ian Horrocks.
\newblock Experience building a large, re-usable medical ontology using a
  description logic with transitivity and concept inclusions.
\newblock In {\em Proceedings of the Workshop on Ontological Engineering, AAAI
  Spring Symposium (AAAI'97)}, Stanford, CA, 1997. AAAI Press.

\bibitem{SoSt08}
Viorica Sofronie-Stokkermans.
\newblock Locality and subsumption testing in {$\mathcal{EL}$} and some of its
  extensions.
\newblock In {\em Proc.\ Advances in Modal Logic ({AiML}'08)}, 2008.

\bibitem{WoZa08}
Frank Wolter and Michael Zakharyaschev.
\newblock Undecidability of the unification and admissibility problems for
  modal and description logics.
\newblock {\em ACM Trans. Comput. Log.}, 9(4), 2008.

\end{thebibliography}

\end{document}